\documentclass{article} 

\usepackage{mathtools}
\usepackage{amsmath,amsthm,amssymb} 
\usepackage{bm}
\usepackage{dsfont}
\usepackage{enumitem}

\usepackage{microtype}
\usepackage{graphicx}    
\usepackage{float}

\usepackage{booktabs} 
\usepackage{tabu}
\usepackage{multirow}
\usepackage{tabularx}
\usepackage{caption}
\usepackage{caption3}
\usepackage{subcaption}
\usepackage{threeparttable}   

\usepackage{hyperref}

\usepackage[accepted]{icml2019}

\makeatletter
\g@addto@macro\normalsize{%
  \setlength{\abovedisplayskip}{-7pt}
  \setlength{\belowdisplayskip}{2pt}
  \setlength{\abovedisplayshortskip}{-7pt}
  \setlength{\belowdisplayshortskip}{3pt}
}

\graphicspath{%
{./figures/},
}

\makeatletter
\providecommand\citet[2][]{%
  \edef\@tempa{#1}
  \citeauthor{#2} (%
    \citeyear{#2}%
    \ifx\@empty\@tempa\else,~#1\fi
  )
}
\makeatother

\setlist[enumerate]{leftmargin=1.5em, topsep=0pt, itemsep=1pt}

\makeatletter
\def\thm@space@setup{%
  \thm@preskip=1em
  \thm@postskip=\thm@preskip 
}
\makeatother

\newcommand*{\sub}[1]{\ensuremath{_\textrm{{\scriptsize #1}}}}
\newcommand*{\super}[1]{\ensuremath{^\textrm{{\scriptsize #1}}}}

\newtheorem{theorem}{Theorem}
\newtheorem{lemma}{Lemma}

\DeclareMathOperator*{\argmin}{arg\,min}

\begin{document}

\twocolumn[
\icmltitle{Classification from Positive, Unlabeled and Biased Negative Data}

\begin{icmlauthorlist}
  \icmlauthor{Yu-Guan Hsieh $^{*}$}{ens}
  \icmlauthor{Gang Niu}{riken}
  \icmlauthor{Masashi Sugiyama}{riken,todai}
\end{icmlauthorlist}

\icmlaffiliation{ens}{\'{E}cole Normale Sup\'{e}rieure, Paris, France}
\icmlaffiliation{riken}{RIKEN, Tokyo, Japan}
\icmlaffiliation{todai}{The University of Tokyo, Tokyo, Japan}

\icmlcorrespondingauthor{Yu-Guan Hsieh}{yu-guan.hsieh@ens.fr}

\icmlkeywords{Positive-Unlabeled Learning, Dataset Shift,
Empirical Risk Minimization, Semi-supervised Learning}

\vskip 0.3in
]

\printAffiliationsAndNotice{}

\begin{abstract}
  In binary classification, there are situations
  where negative (N) data are too diverse to
  be fully labeled and
  we often resort to positive-unlabeled (PU)
  learning in these scenarios.
  However, collecting a non-representative N set that
  contains only a small portion of all possible N
  data can often be much easier in practice.
  This paper studies a novel classification
  framework which incorporates such biased N (bN) data
  in PU learning.
  We provide a method based on empirical risk minimization
  to address this PUbN classification problem.
  Our approach can
  be regarded as a novel example-weighting
  algorithm, with
  the weight of each example computed through
  a preliminary step that draws inspiration from PU learning.
  We also derive an estimation error bound for the proposed
  method.
  Experimental results demonstrate the effectiveness
  of our algorithm in not only PUbN learning scenarios
  but also ordinary PU learning scenarios on several
  benchmark datasets.
\end{abstract}

\section{Introduction}

In conventional binary classification, examples are labeled as
either positive (P) or negative (N), and we train a classifier on
these labeled examples.
On the contrary, positive-unlabeled (PU) learning addresses the
problem of learning a classifier from P and unlabeled (U)
data, without the need of explicitly identifying N data
\citep{Elkan2008LearningCF, Ward2009PresenceonlyDA}.

PU learning finds its usefulness in many real-world problems.
For example, in one-class remote sensing classification
\citep{li2011positive},
we seek to extract a specific land-cover class from an image.
While it is easy to label examples of this specific land-cover
class of interest, examples not belonging to this
class are too diverse to be exhaustively annotated.
The same problem arises in text classification, as it is
difficult or even impossible to compile a set of N
samples that provides a comprehensive characterization of
everything that is not in the P class
\citep{liu2003building, fung2006text}.
Besides, PU learning has also been applied to other domains
such as outlier detection \citep{hido2008inlier, scott2009novelty},
medical diagnosis \citep{zu2011learn},
or time series classification \citep{nguyen2011positive}.

By carefully examining the above examples, we find out that
the most difficult step is often to collect a fully representative
N set, whereas only labeling a small portion of all possible N
data is relatively easy.
Therefore, in this paper, we propose to study the problem of
learning from P, U and biased N (bN) data, which we name PUbN learning
hereinafter.
We suppose that in addition to P and U data, we also gather
a set of bN samples, governed by a distribution
distinct from the true N distribution.
As described previously, this can be viewed as an extension
of PU learning, but such bias may also occur naturally in
some real-world scenarios.
For instance, let us presume that we would like to judge
whether a subject
is affected by a particular disease based on the result
of a physical examination.
While the data collected from the patients represent rather
well the P distribution,
healthy subjects that request the examination
are in general
biased with respect to the whole healthy
subject population.

We are not the first to be interested in learning with
bN data.
In fact, both \citet{li2010neg} and \citet{fei2015social}
attempted to solve similar problems in the context of text
classification.
\citet{li2010neg} simply discarded N samples and
performed ordinary PU classification.
It was also mentioned in the paper that bN data
could be harmful.
\citet{fei2015social} adopted another strategy.
The authors considered even gathering unbiased U data
is difficult and learned the classifier from only
P and bN data.
However, their method is specific to text classification
because it relies on the use of effective similarity
measures to evaluate similarity between documents
(refer to Supplementary Material \ref*{app: cbs}
for a deeper discussion
and an empirical comparison with our method).
Therefore, our work differs from these two in that
the classifier is trained simultaneously on P, U and bN
data, without resorting to domain-specific knowledge.
The presence of U data allows us to address the problem
from a statistical viewpoint, and thus the proposed method can
be applied to any PUbN learning problem in principle.

In this paper,
we develop an empirical risk minimization-based algorithm
that combines both PU learning and importance weighting
to solve the PUbN classification problem.
We first estimate the probability that an example
is sampled into the P or the bN set.
Based on this estimate, we regard bN and U data
as N examples with instance-dependent weights.
In particular, we assign larger weights to U
examples that we believe to appear less often in the
P and bN sets.
P data are treated as P examples with unity weight
but also as N examples with usually small or zero weight
whose actual value depends on the same estimate.

The contributions of the paper are three-fold:

\begin{enumerate}
  \item
    We formulate the PUbN learning problem as an extension
    of PU learning and propose an empirical risk minimization-based
    method to address the problem.
    We also theoretically establish an estimation error bound
    for the proposed method.
  \item
    We experimentally demonstrate that the classification performance
    can be effectively improved thanks to the use of
    bN data during training.
    In other words, PUbN learning yields better performance than
    PU learning.
  \item
    Our method can be easily adapted to ordinary PU learning.
    Experimentally we show that the resulting algorithm allows
    us to obtain new
    state-of-the-art results on several PU learning tasks.
\end{enumerate}

\textbf{Relation with Semi-supervised Learning.}\enspace
With P, N and U data available for training,
our problem setup may seem similar to that of
semi-supervised learning
\citep{Chapelle:2010:SL:1841234, oliver2018realistic}.
Nonetheless, in our case, N data are biased and often
represent only a small portion of the whole N distribution.
Therefore, most of the existing methods designed for
the latter cannot be directly applied to
the PUbN classification problem.
Furthermore, our focus is on deducing a risk estimator
using the three sets of data,
with U data in particular used to compensate the sampling
bias in N data.
On the other hand, in semi-supervised learning
the main concern is often how U data
can be utilized for regularization
\citep{grandvalet2005semi, belkin2006manifold,
miyato2016distributional, laine2017temporal}.
The two should be compatible and we believe
adding such regularization to our algorithm can
be beneficial in many cases.

\textbf{Relation with Dataset Shift.}\enspace
PUbN learning can also be viewed as a special case of
dataset shift%
\footnote{
  Dataset shift refers to any case where training
  and test distributions differ.
  The term sample selection bias
  \citep{heckman1979sample, zadrozny2004learning}
  is sometimes used to describe the same thing.
  However, strictly speaking, sample selection bias
  actually refers to the case where training
  instances are first drawn from the test distributions
  and then a subset of these data is systematically
  discarded due to a particular mechanism.
}
\citep{quionero2009dataset}
if we consider that P and bN data are drawn
from the training distribution while
U data are drawn from the test distribution.
Covariate shift
\citep{shimodaira2000improving, book:Sugiyama+Kawanabe:2012}
is another special case of dataset shift that has been
studied intensively.
In the covariate shift problem setting, training
and test distributions have the same class conditional
distribution and only differ in the marginal distribution
of the independent variable.
One popular approach to tackle this problem
is to reweight each training example according
to the ratio of the test density to the training density
\citep{huang2007correcting, sugiyama2008direct}.
Nevertheless, 
simply training a classifier on a reweighted version of
the labeled set is not sufficient in our case
since there may be examples with zero probability to be
labeled, and it is therefore essential to involve
U samples in the second step of the proposed algorithm.
It is also important to notice that 
the problem of PUbN learning is
intrinsically different from that of covariate shift and
neither of the two is a special case of the other.

Finally, source component shift
\citep{quionero2009dataset} is also related.
It assumes that data are generated from several
different sources and the proportions of these
sources may vary between training and test times.
In many practical situations, this is indeed what causes
our collected N data to be biased.
However, its definition is so general that
we are not aware of any universal method which
addresses this problem without explicit model assumptions on
data distribution.

\section{Problem Setting}

In this section, we briefly review the formulations of PN,
PU and PNU classification and introduce the problem of
learning from P, U and bN data.

\subsection{Standard Binary Classification}

Let $\bm{x}\in\mathbb{R}^d$ and $y\in\{+1, -1\}$ be
random variables following an unknown probability distribution
with density
$p(\bm{x}, y)$. Let $g: \mathbb{R}^d \rightarrow \mathbb{R}$
be an arbitrary decision function for binary classification
and $\ell: \mathbb{R}\rightarrow\mathbb{R}_+$ be a loss
function of margin $yg(\bm{x})$ that usually takes a
small value for a large margin.
The goal of binary classification is to find $g$ that minimizes
the classification risk:

\begin{equation} \label{eq: risk}
  R(g) = \mathbb{E}_{(\bm{x},y) \sim p(\bm{x}, y)}[\ell(yg(\bm{x}))],
\end{equation}

where $\mathbb{E}_{(\bm{x}, y)\sim p(\bm{x}, y)}[\cdot]$ denotes
the expectation over the joint distribution $p(\bm{x}, y)$.
When we care about classification accuracy, $\ell$ is the
zero-one loss $\ell_{01}(z) = (1-\mathrm{sign}(z))/2$.
However, for ease of optimization, $\ell_{01}$ is often
substituted with a surrogate loss such as the sigmoid loss
$\ell\sub{sig}(z) = 1/(1+\exp(z))$ or the logistic loss
$\ell\sub{log}(z) = \ln(1+\exp(-z))$ during learning.

In standard supervised learning scenarios (PN classification),
we are given P and N data that are sampled
independently from
$p\sub{P}(\bm{x}) = p(\bm{x}\mid y=+1)$ and
$p\sub{N}(\bm{x}) = p(\bm{x}\mid y=-1)$
as $\mathcal{X}\sub{P} = \{\bm{x}_i\super{P}\}_{i=1}^{n\sub{P}}$
and $\mathcal{X}\sub{N} = \{\bm{x}_i\super{N}\}_{i=1}^{n\sub{N}}$.
Let us denote by
$R\sub{P}^+(g) = \mathbb{E}_{\bm{x} \sim p\sub{P}(\bm{x})}[\ell(g(\bm{x}))]$,
$R\sub{N}^-(g) = \mathbb{E}_{\bm{x} \sim p\sub{N}(\bm{x})}[\ell(-g(\bm{x}))]$
partial risks and $\pi = p(y=1)$ the P prior.
We have the equality $R(g) = \pi R\sub{P}^+(g)+(1-\pi) R\sub{N}^-(g)$.
The classification risk \eqref{eq: risk} can then be
empirically approximated from data by

\[\hat{R}\sub{PN}(g) = \pi \hat{R}\sub{P}^+(g)+(1-\pi) \hat{R}\sub{N}^-(g),\]

where $\hat{R}\sub{P}^+(g) =
\frac{1}{n\sub{P}} \sum_{i=1}^{n\sub{P}} \ell(g(\bm{x}_i\super{P}))$
and $\hat{R}\sub{N}^-(g) =
\frac{1}{n\sub{N}} \sum_{i=1}^{n\sub{N}} \ell(-g(\bm{x}_i\super{N}))$.
By minimizing $\hat{R}\sub{PN}(g)$ we obtain the ordinary
empirical risk minimizer $\hat{g}\sub{PN}$.

\subsection{PU Classification}

In PU classification, instead of N data $\mathcal{X}\sub{N}$
we have only access to
$\mathcal{X}\sub{U} = \{x_i\super{U}\}_{i=1}^{n\sub{U}} \sim p(\bm{x})$
a set of U samples drawn from the marginal density $p(\bm{x})$.
Several effective algorithms have been designed to
address this problem.
\citet{liu2002partially} proposed the S-EM approach that first
identifies reliable N data in the U set and then run
the Expectation-Maximization (EM) algorithm to build the final classifier.
The biased support vector machine (Biased SVM) introduced by
\citet{liu2003building}
regards U samples as N samples with smaller weights.
\citet{mordelet2014bagging} solved the PU problem by
aggregating classifiers trained to discriminate P data from
a small random subsample of U data.
An ad hoc algorithm designed for linear classifiers,
treating the U set as an N set influenced by label noise,
was proposed in \cite{ijcai2018-373}. 

Recently, attention has also been paid on the unbiased risk
estimator proposed by \citeauthor{du2014analysis}
\yrcite{du2014analysis, du2015convex}.
The key idea is to use the following equality:

\[(1-\pi)R\sub{N}^-(g) = R\sub{U}^-(g)-\pi R\sub{P}^-(g),\]

where
$R\sub{U}^-(g) = \mathbb{E}_{x \sim p(\bm{x})}[\ell(-g(\bm{x}))]$
and
$R\sub{P}^-(g) = \mathbb{E}_{x \sim p\sub{P}(\bm{x})}[\ell(-g(\bm{x}))].$
This equality is acquired by exploiting the fact
$p(\bm{x}) = \pi p\sub{P}(\bm{x}) + (1-\pi) p\sub{N}(\bm{x})$.
As a result, we can approximate the classification risk
\eqref{eq: risk} by

\begin{equation} \label{eq: pu risk estimator}
  \hat{R}\sub{PU}(g)
  = \pi\hat{R}\sub{P}^+(g)-\pi\hat{R}\sub{P}^-(g)+\hat{R}\sub{U}^-(g),
\end{equation}

where $\hat{R}\sub{P}^-(g) =
\frac{1}{n\sub{P}} \sum_{i=1}^{n\sub{P}} \ell(-g(\bm{x}_i\super{P}))$
and $\hat{R}\sub{U}^-(g) =
\frac{1}{n\sub{U}} \sum_{i=1}^{n\sub{U}} \ell(-g(\bm{x}_i\super{U}))$.
We then minimize $\hat{R}\sub{PU}(g)$ to obtain another
empirical risk minimizer $\hat{g}\sub{PU}$.
Note that as the loss is always positive, the classification risk
\eqref{eq: risk} that $\hat{R}\sub{PU}(g)$ approximates is also positive.
However, \citet{kiryo2017positive} pointed out that when the
model of $g$ is too flexible, that is, when the function class
$\mathcal{G}$ is too large, $\hat{R}\sub{PU}(\hat{g}\sub{PU})$
indeed goes negative and the model severely overfits
the training data. To alleviate overfitting, the authors observed
that $R\sub{U}^-(g)-\pi R\sub{P}^-(g) = (1-\pi)R\sub{N}^-(g) \ge 0$
and proposed the non-negative risk estimator for PU learning:

\begin{equation} \label{eq: nnpu risk estimator}
  \tilde{R}\sub{PU}(g) =
  \pi\hat{R}\sub{P}^+(g)
  + \max\{0, \hat{R}\sub{U}^-(g)-\pi\hat{R}\sub{P}^-(g)\}.
\end{equation}

In terms of implementation, stochastic optimization was used
and when $r = \hat{R}\sub{U}^-(g)-\pi\hat{R}\sub{P}^-(g)$ becomes
smaller than some threshold value $-\beta$ for a mini-batch,
they performed a step of gradient ascent
along $\nabla r$ to make the mini-batch less overfitted.

\subsection{PNU Classification} \label{subsec: pnu}

In semi-supervised learning (PNU classification), P, N and U
data are all available.
An abundance of works have been dedicated to solving this problem.
Here we in particular introduce the PNU risk estimator
proposed by \citet{sakai2016semi}.
By directly leveraging U data for risk estimation,
it is the most comparable to our method.
The PNU risk is simply defined as a linear combination of PN
and PU/NU risks. Let us just consider the case where PN
and PU risks are combined, then for some $\gamma\in[0, 1]$,
the PNU risk estimator is expressed as

\begin{align} \label{eq: pnu risk estimator}
  \hat{R}\sub{PNU}^{\gamma}(g)
  &= \gamma\hat{R}\sub{PN}(g) + (1-\gamma)\hat{R}\sub{PU}(g)\nonumber\\
  &= \pi\hat{R}\sub{P}^+(g) + \gamma(1-\pi)\hat{R}\sub{N}^-(g)\nonumber\\
  &\enspace + (1-\gamma)(\hat{R}\sub{U}^-(g)-\pi\hat{R}\sub{P}^-(g)).
\end{align}

We can again consider the non-negative correction by forcing
the term
$\gamma(1-\pi)\hat{R}\sub{N}^-(g) +
 (1-\gamma)(\hat{R}\sub{U}^-(g)-\pi\hat{R}\sub{P}^-(g))$
to be non-negative. In the rest of the paper, we refer to the resulting
algorithm as non-negative PNU (nnPNU) learning
(see Supplementary Material \ref*{app: pnu-alt}
for an alternative definition
of nnPNU and the corresponding results).

\subsection{PUbN Classification}

In this paper, we study the problem of PUbN learning.
It differs from usual semi-supervised learning
in the fact that labeled N data are not fully representative
of the underlying N distribution $p\sub{N}(\bm{x})$.
To take this point into account, we introduce a
latent random variable $s$ and consider the
joint distribution $p(\bm{x}, y, s)$
with constraint $p(s=+1\mid\bm{x}, y=+1)=1$.
Equivalently, $p(y=-1\mid\bm{x}, s=-1)=1$.
Let $\rho = p(y=-1, s=+1)$.
Both $\pi$ and $\rho$ are assumed known throughout the paper.
In practice they often need to be estimated from data 
\citep{jain2016estimating, ramaswamy2016mixture, 
Plessis:2017:CEL:3085961.3085999}.
In place of ordinary N data we collect a set of
bN samples 

\[\mathcal{X}\sub{bN} = \{\bm{x}_i\super{bN}\}_{i=1}^{n\sub{bN}}
  \sim p\sub{bN}(\bm{x}) = p(\bm{x}|y=-1, s=+1).\]

For instance, in text classification, if our bN data is composed
of a small set of all possible N topics, $s=+1$ means that a sample
is either from these topics that make up the bN set or in the P class.
The goal remains the same: we would like to minimize
the classification risk $\eqref{eq: risk}$.

\section{Method}

In this section, we propose a risk estimator for PUbN classification
and establish an estimation error bound for the proposed method.
Finally we show how our method can be applied to PU
learning as a special case when no bN data are available.

\subsection{Risk Estimator}

Let
$R\sub{bN}^-(g) =
\mathbb{E}_{x \sim p\sub{bN}(\bm{x})}[\ell(-g(\bm{x}))]$
and
$R_{s=-1}^-(g) = \mathbb{E}_{x \sim p(\bm{x}\mid s=-1)}[\ell(-g(\bm{x}))]$.
Since
$p(\bm{x}, y=-1) = p(\bm{x}, y=-1, s=+1) + p(\bm{x}, s=-1)$,
we have 

\begin{equation} \label{eq: PUbN risk}
  R(g) = \pi R\sub{P}^+(g) + \rho R\sub{bN}^-(g) + (1-\pi-\rho)R_{s=-1}^-(g).
\end{equation}

The first two terms on the right-hand side of the equation can
be approximated directly from data by
$\hat{R}\sub{P}^+(g)$
and
$\hat{R}\sub{bN}^-(g)
  = \frac{1}{n\sub{bN}} \sum_{i=1}^{n\sub{bN}} \ell(-g(\bm{x}_i\super{bN}))$.
We therefore focus on the third term
$\bar{R}_{s=-1}^-(g) = (1-\pi-\rho)R_{s=-1}^-(g)$.
Our approach is mainly based on the following theorem.
We relegate all proofs to the Supplementary Material.

\begin{theorem}
  Let $\sigma(\bm{x}) = p(s=+1\mid\bm{x})$. For all $\eta\in[0, 1]$
  and $h: \mathbb{R}^d \rightarrow [0, 1]$ satisfying
  the condition $h(\bm{x})>\eta\Rightarrow\sigma(\bm{x})>0$,
  the risk $\bar{R}_{s=-1}^-(g)$ can be expressed as

  \vspace{0.8em}
  \resizebox{\linewidth}{!}{
  \begin{minipage}{\linewidth}
  \begin{align} \label{eq: decompose}
    \bar{R}_{s=-1}^-(g)
    &= \mathbb{E}_{\bm{x}\sim p(\bm{x})}[
        \mathds{1}_{h(\bm{x})\le\eta}
        \thinspace\ell(-g(\bm{x}))
        (1 - \sigma(\bm{x}))]\nonumber\\
    &\enspace
     + \pi\thinspace
     \mathbb{E}_{\bm{x}\sim p\sub{\emph{P}}(\bm{x})}\left[
        \mathds{1}_{h(\bm{x})>\eta}
        \thinspace\ell(-g(\bm{x}))
        \frac{1-\sigma(\bm{x})}{\sigma(\bm{x})}\right]\nonumber\\
    &\enspace
     + \rho\thinspace
     \mathbb{E}_{\bm{x}\sim p\sub{\emph{bN}}(\bm{x})}\left[
        \mathds{1}_{h(\bm{x})>\eta}
        \thinspace\ell(-g(\bm{x}))
        \frac{1-\sigma(\bm{x})}{\sigma(\bm{x})}\right].\nonumber\\
  \end{align}
  \end{minipage}}
\end{theorem}

\vspace{-0.8em}
In the theorem, $\bar{R}_{s=-1}^-(g)$ is decomposed into three
terms, and when the expectation is substituted with the average
over training samples, these three terms are approximated
respectively using
data from $\mathcal{X}\sub{U}$, $\mathcal{X}\sub{P}$ and
$\mathcal{X}\sub{bN}$.
The choice of $h$ and $\eta$ is thus very crucial
because it determines what each of the three terms
tries to capture in practice.
Ideally, we would like $h$ to be an approximation of $\sigma$.
Then, for $\bm{x}$ such that $h(\bm{x})$ is close to 1, $\sigma(\bm{x})$
is close to 1, so the last two terms on the right-hand side of
the equation can be reasonably evaluated using $\mathcal{X}\sub{P}$
and $\mathcal{X}\sub{bN}$ (i.e., samples drawn from $p(\bm{x}\mid s=+1)$).
On the contrary, if $h(\bm{x})$ is small, $\sigma(\bm{x})$ is small
and such samples can be hardly found in $\mathcal{X}\sub{P}$
or $\mathcal{X}\sub{bN}$.
Consequently the first term appeared in the decomposition
is approximated with the help of $\mathcal{X}\sub{U}$.
Finally, in the empirical risk minimization paradigm, $\eta$ becomes
a hyperparameter that controls how important U data is against
P and bN data when we evaluate $\bar{R}_{s=-1}^-(g)$.
The larger $\eta$ is, the more attention we would pay to U data.


One may be curious about why we do not simply approximate the whole risk
using only U samples, that is, set $\eta$ to 1.
There are two main reasons.
On one hand, if we have a very small U set, which
means $n\sub{U} \ll n\sub{P}$ and $n\sub{U} \ll n\sub{bN}$, approximating
a part of the risk with labeled samples should help us reduce the
estimation
error. This may seem unrealistic but sometimes unbiased U samples
can also be difficult to collect \citep{ishida2018binary}.
On the other hand, more importantly, we have empirically observed that
when the model of $g$ is highly flexible, 
even a sample regarded as N with small weight
gets classified as N in the latter stage of
training and performance of the resulting classifier can thus
be severely degraded.
Introducing $\eta$ alleviates this problem by avoiding treating
all U data as N samples.

As $\sigma$ is not available in reality, we propose to replace
$\sigma$ by its estimate $\hat{\sigma}$ in
\eqref{eq: decompose}.
We further substitute $h$ with the same estimate and obtain the following
expression:
  
\vspace{0.7em}
\resizebox{\linewidth}{!}{
\begin{minipage}{\linewidth}
\begin{align*}
  \bar{R}_{s=-1, \eta, \hat{\sigma}}^-(g)
  &= \mathbb{E}_{\bm{x}\sim p(\bm{x})}[
      \mathds{1}_{\hat{\sigma}(\bm{x})\le\eta}
      \thinspace\ell(-g(\bm{x}))
      (1 - \hat{\sigma}(\bm{x}))]\\
  &\enspace
   + \pi\thinspace
    \mathbb{E}_{\bm{x}\sim p\sub{P}(\bm{x})}\left[
        \mathds{1}_{\hat{\sigma}(\bm{x})>\eta}
      \thinspace\ell(-g(\bm{x}))
      \frac{1-\hat{\sigma}(\bm{x})}{\hat{\sigma}(\bm{x})}\right]\\
  &\enspace
   + \rho\thinspace
    \mathbb{E}_{\bm{x}\sim p\sub{bN}(\bm{x})}\left[
      \mathds{1}_{\hat{\sigma}(\bm{x})>\eta}
      \thinspace\ell(-g(\bm{x}))
      \frac{1-\hat{\sigma}(\bm{x})}{\hat{\sigma}(\bm{x})}\right].
\end{align*}
\end{minipage}}
\vspace{0.2em}

We notice that $\bar{R}_{s=-1,\eta,\hat{\sigma}}$ depends both on
$\eta$ and $\hat{\sigma}$. It can be directly approximated
from data by

\vspace{0.7em}
\resizebox{\linewidth}{!}{
\begin{minipage}{\linewidth}
\begin{align*}
  \hat{\bar{R}}_{s=-1,\eta,\hat{\sigma}}(g)
  &= 
  \frac{1}{n\sub{U}}\sum_{i=1}^{n\sub{U}}
  \left[
    \mathds{1}_{\hat{\sigma}(\bm{x}_i\super{U})\le\eta}
    \thinspace\ell(-g(\bm{x}_i\super{U}))
    (1-\hat{\sigma}(\bm{x}_i\super{U}))
  \right]\\
  &\enspace
  + \frac{\pi}{n\sub{P}}\sum_{i=1}^{n\sub{P}}
    \left[
      \mathds{1}_{\hat{\sigma}(\bm{x}_i\super{P})>\eta}
      \thinspace\ell(-g(\bm{x}_i\super{P}))
      \frac{1-\hat{\sigma}(\bm{x}_i\super{P})}
      {\hat{\sigma}(\bm{x}_i\super{P})}
    \right]\\
  &\enspace
  + \frac{\rho}{n\sub{bN}}\sum_{i=1}^{n\sub{bN}}
    \left[
      \mathds{1}_{\hat{\sigma}(\bm{x}_i\super{bN})>\eta}
      \thinspace\ell(-g(\bm{x}_i\super{bN}))
      \frac{1-\hat{\sigma}(\bm{x}_i\super{bN})}
          {\hat{\sigma}(\bm{x}_i\super{bN})}
    \right].
\end{align*}
\end{minipage}
}
\vspace{0.2em}

We are now able to derive the empirical version of
Equation \eqref{eq: PUbN risk} as

\begin{equation} \label{eq: PUbN risk estimator}
  \hat{R}_{\textrm{PUbN},\eta,\hat{\sigma}}(g) = 
  \pi\hat{R}\sub{P}^+(g) + \rho\hat{R}\sub{bN}^-(g) 
  + \hat{\bar{R}}_{s=-1, \eta, \hat{\sigma}}^-(g).
\end{equation}

\subsection{Practical Implementation}

To complete our algorithm, we need to be able to
estimate $\sigma$ and find appropriate $\eta$.
Given that the value of $\eta$ can be hard to tune,
we introduce another intermediate hyperparameter $\tau$
and choose $\eta$ such that
$\#\{x\in\mathcal{X}\sub{U} \mid \hat{\sigma}(x)\le\eta\}
= \lfloor\tau(1-\pi-\rho)n\sub{U}\rfloor$,
where $\lfloor\cdot\rfloor$ is the floor function.
The number $\tau(1-\pi-\rho)$ is then the portion of unlabeled
samples that are involved in the second step of our algorithm.
Intuitively, we can set a higher $\tau$ and include more
U samples in the minimization of \eqref{eq: PUbN risk estimator}
when we have a good estimate $\hat{\sigma}$
and otherwise we should prefer a smaller $\tau$ to reduce
the negative effect that can be caused by the use of
$\hat{\sigma}$ of poor quality.
The use of validation data to select the final $\tau$ should
also be prioritized as what we do in the experimental part.

\textbf{Estimating $\mathbf{\sigma}$.}\enspace
If we regard $s$ as a class label, the problem of
estimating $\sigma$ is then equivalent to
training a probabilistic classifier separating the
classes with $s=+1$ and $s=-1$.
Upon noting that
$(\pi+\rho)\mathbb{E}_{\bm{x}\sim p(\bm{x}\mid s=+1)}[\ell(\epsilon g(x))]
= \pi\mathbb{E}_{\bm{x}\sim p\sub{P}(\bm{x})}[\ell(\epsilon g(x))]
+ \rho\mathbb{E}_{\bm{x}\sim p\sub{bN}(\bm{x})}[\ell(\epsilon g(x))]$
for $\epsilon\in\{+1, -1\}$, it is straightforward to apply nnPU
learning with availability of $\mathcal{X}\sub{P}$, $\mathcal{X}\sub{bN}$ and
$\mathcal{X}\sub{U}$ to minimize
$\mathbb{E}_{(\bm{x}, s)\sim p(\bm{x},s)}[\ell(sg(\bm{x}))]$.
In other words, here we regard $\mathcal{X}\sub{P}$ and
$\mathcal{X}\sub{bN}$ as P
and $\mathcal{X}\sub{U}$ as U, and attempt to solve a PU learning problem
by applying nnPU\@.
Since we are interested in the class-posterior probabilities, we minimize
the risk with respect to the
logistic loss and apply the sigmoid function to the output of
the model to get $\hat{\sigma}(\bm{x})$.
However, the above risk estimator accepts any reasonable $\hat{\sigma}$
and we are not limited to using nnPU for computing $\hat{\sigma}$.
For example, the least-squares fitting approach proposed by
\citet{Kanamori2009ALA} for direct density ratio estimation can also
be adapted to solving the problem.

To handle large datasets, it is preferable to adopt
stochastic optimization algorithms to minimize 
$\hat{R}_{\textrm{PUbN},\eta,\hat{\sigma}}(g)$.

{\setlength{\textfloatsep}{8pt}
\begin{algorithm}[tb]
   \caption{PUbN Classification}
\begin{algorithmic}[1]
  \STATE {\bfseries Input:} data
    $(\mathcal{X}\sub{P}, \mathcal{X}\sub{bN}, \mathcal{X}\sub{U})$,
    hyperparameter $\tau$
  \STATE {\bfseries Step 1:}
    \STATE Compute $\hat{\sigma}$ by minimizing an nnPU risk
    involving $\mathcal{X}\sub{P}$, $\mathcal{X}\sub{bN}$
    as P data and $\mathcal{X}\sub{U}$ as U data
  \STATE {\bfseries Step 2:}
    \STATE Initialize model parameter $\theta$ of $g$
    \STATE Choose $\mathcal{A}$ a SGD-like stochastic optimization
    algorithm \hspace{-1em}
    \STATE Set $\eta$ such that

    \[\#\{x\in\mathcal{X}\sub{U} \mid \hat{\sigma}(x)\le\eta\}
    = \lfloor\tau(1-\pi-\rho)n\sub{U}\rfloor\]

    \FOR{$i=1 \ldots$}
    \STATE Shuffle
    $(\mathcal{X}\sub{P}, \mathcal{X}\sub{bN}, \mathcal{X}\sub{U})$
    into M mini-batches
    \FOR{each mini-batch
    $(\mathcal{X}\sub{P}^j, \mathcal{X}\sub{bN}^j, \mathcal{X}\sub{U}^j)$}
    \STATE Compute the corresponding
    $\hat{R}_{\textrm{PUbN},\eta,\hat{\sigma}}(g)$
    \STATE Use $\mathcal{A}$ to update $\theta$ with the gradient information
    $\nabla_{\theta}\hat{R}_{\textrm{PUbN},\eta,\hat{\sigma}}(g)$
    \ENDFOR
    \ENDFOR
    \STATE {\bfseries Return:} $\theta$ minimizing the validation loss
\end{algorithmic}
\end{algorithm}}

\subsection{Estimation Error Bound}

Here we establish an estimation error bound for the proposed method.
Let $\mathcal{G}$ be the function class from which we find a
function. The Rademacher complexity
of $\mathcal{G}$ for the samples of size $n$ drawn
from $q(\bm{x})$ is defined as

\[\mathfrak{R}_{n,q}(\mathcal{G})
  = \mathbb{E}_{\mathcal{X}\sim q^n}
    \mathbb{E}_{\xi}\left[\sup_{g\in\mathcal{G}}
    \frac{1}{n}\sum_{x_i\in\mathcal{X}}\xi_i g(\bm{x}_i)\right],\]

where $\mathcal{X}=\{\bm{x}_1, \ldots, \bm{x}_n\}$ and
$\xi=\{\xi_1, \ldots, \xi_n\}$ with each $\bm{x}_i$
drawn from $q(\bm{x})$ and $\xi_i$ as a Rademacher variable
\citep{mohri2012foundations}.
In the following we will assume that $\mathfrak{R}_{n,q}(\mathcal{G})$
vanishes asymptotically as $n\rightarrow\infty$.
This holds for most of the common choices of
$\mathcal{G}$ if proper regularization is considered
\citep{bartlett2002rademacher, golowich2018size}.
Assume additionally the existence of $C_{g} > 0$
such that $\sup_{g\in\mathcal{G}}\|g\|_{\infty}\le C_g$
as well as $C_{\ell} > 0$ such that
$\sup_{|z|\le C_g}\ell(z)\le C_{\ell}$.
We also assume that $\ell$ is Lipschitz continuous
on the interval $[-C_g, C_g]$ with a Lipschitz constant $L_{\ell}$.


\begin{theorem}\label{th: bound}
  Let $g^*=\argmin_{g\in\mathcal{G}}R(g)$ be the true risk minimizer
  and
  $\hat{g}_{\emph{PUbN}, \eta, \hat{\sigma}}
  = \argmin_{g\in\mathcal{G}}\hat{R}_{\emph{PUbN}, \eta, \hat{\sigma}}(g)
  $ be the PUbN empirical risk minimizer.
  We suppose that $\hat{\sigma}$ is a fixed function independent
  of data used to compute $\hat{R}_{\emph{PUbN}, \eta, \hat{\sigma}}(g)$
  and $\eta\in(0, 1]$.
  Let $\zeta = p(\hat{\sigma}(\bm{x})\le\eta)$ and
  $\epsilon = \mathbb{E}_{\bm{x}\sim p(\bm{x})}
    [|\hat{\sigma}(\bm{x})-\sigma(\bm{x})|^2]$.
  Then for any $\delta>0$, with probability at least $1-\delta$,

\vspace{0.7em}
\resizebox{\linewidth}{!}{
\begin{minipage}{\linewidth}
\begin{align*}
  & R(\hat{g}_{\emph{PUbN}, \eta, \hat{\sigma}}) - R(g^*) \\
  &\enspace\le
  4 L_\ell \mathfrak{R}_{n\sub{\emph{U}}, p}(\mathcal{G})
  + \frac{4\pi L_\ell}{\eta}
    \mathfrak{R}_{n\sub{\emph{P}}, p\sub{\emph{P}}}(\mathcal{G})
  + \frac{4\rho L_\ell}{\eta}
    \mathfrak{R}_{n\sub{\emph{bN}},p\sub{\emph{bN}}}(\mathcal{G}) \\
  &\quad
  + 2C_\ell\sqrt{\frac{\ln(6/\delta)}{2n\sub{\emph{U}}}}
  + \frac{2\pi C_\ell}{\eta}
    \sqrt{\frac{\ln(6/\delta)}{2n\sub{\emph{P}}}}
  + \frac{2\rho C_\ell}{\eta}
    \sqrt{\frac{\ln(6/\delta)}{2n\sub{\emph{bN}}}}\\
  &\quad
  + 2C_\ell\sqrt{\zeta\epsilon}
  + \frac{2C_\ell}{\eta}\sqrt{(1-\zeta)\epsilon}.
\end{align*}
\end{minipage}}
\end{theorem}
\vspace{-0.6em}

Theorem \ref{th: bound} combined with
the Borel-Cantelli lemma implies that as
$n\sub{P}\rightarrow\infty$,
$n\sub{bN}\rightarrow\infty$ and $n\sub{U}\rightarrow\infty$,
the inequality
$\limsup
R(\hat{g}_{\textrm{PUbN}, \eta, \hat{\sigma}}) - R(g^*)
\le 
2C_\ell\sqrt{\zeta\epsilon}
+ 2(C_\ell/\eta)\sqrt{(1-\zeta)\epsilon}$ holds almost surely.
Furthermore, if there is $C_{\mathcal{G}}>0$ such that
$\mathfrak{R}_{n,q}(\mathcal{G})\le C_{\mathcal{G}}/\sqrt{n}$
\footnote{
  For instance, this holds for linear-in-parameter model class
  $\mathcal{F} = \{f(\bm{x})=\bm{w}^{\top}\phi(\bm{x})\mid
    \|\bm{w}\|\le C_{\bm{w}}, \|\phi\|_{\infty}\le C_{\phi}\}$,
  where $C_{\bm{w}}$ and $C_{\phi}$ are positive constants
  \citep{mohri2012foundations}.
},
the convergence of
$[(R(\hat{g}_{\textrm{PUbN}, \eta, \hat{\sigma}}) - R(g^*))
-
(2C_\ell\sqrt{\zeta\epsilon}
+ 2(C_\ell/\eta)\sqrt{(1-\zeta)\epsilon})]^+$ to $0$ is in
$\mathcal{O}_p(1/\sqrt{n\sub{P}}+1/\sqrt{n\sub{bN}}+1/\sqrt{n\sub{U}})$,
where $\mathcal{O}_p$ denotes the order in probability and
$[\cdot]^+ = \max\{0, \cdot\}$.
As for $\epsilon$, knowing that $\hat{\sigma}$ is also estimated
from data in practice
\footnote{
  These data, according to theorem \ref{th: bound}, must be different
  from those used to evaluate
  $\hat{R}_{\textrm{PUbN}, \eta, \hat{\sigma}}(g)$.
  This condition is however violated in most of our
  experiments. See Supplementary Material \ref*{app: sig sep}
  for more discussion.
}, apparently its value depends
on both the estimation algorithm and the number of samples
that are involved in the estimation process.
For example, in our approach we applied nnPU with the logistic loss
to obtain $\hat{\sigma}$, so the excess risk can be written as
$\mathbb{E}_{\bm{x}\sim p(\bm{x})}
  \mathrm{KL}(\sigma(\bm{x})\|\hat{\sigma}(\bm{x}))$,
where by abuse of notation
$\mathrm{KL}(p\|q) = p\ln(p/q) + (1-p)\ln((1-p)/(1-q))$
denotes the KL divergence between two Bernouilli
distributions with parameters respectively $p$ and $q$.
It is known that 
$\epsilon = \mathbb{E}_{\bm{x}\sim p(\bm{x})}
  [|\hat{\sigma}(\bm{x})-\sigma(\bm{x})|^2]
\le(1/2)\mathbb{E}_{\bm{x}\sim p(\bm{x})}
  \mathrm{KL}(\sigma(\bm{x})\|\hat{\sigma}(\bm{x}))$
\citep{zhang2004statistical}.
The excess risk itself can be decomposed into the sum of
the estimation error and the approximation error.
\citet{kiryo2017positive} showed that under mild assumptions
the estimation error part converges to zero when the
sample size increases to infinity in nnPU learning.
It is however impossible to get rid of the approximation error
part which is fixed once we fix the function class $\mathcal{G}$.
To circumvent this problem, we can either resort to
kernel-based methods with universal kernels
\citep{zhang2004statistical} or simply
enlarge the function class when we get more samples.

\subsection{PU Learning Revisited}

In PU learning scenarios, we only have P and U data and
bN data are not available.
Nevertheless, if we let $y$ play the role of $s$ and ignore all the
terms related to bN data, our algorithm is naturally applicable
to PU learning.
Let us name the resulting algorithm PUbN\textbackslash N, then

\[\hat{R}_{\textrm{PUbN\textbackslash N},\eta,\hat{\sigma}}(g) = 
  \pi\hat{R}\sub{P}^+(g) + \hat{\bar{R}}_{y=-1, \eta, \hat{\sigma}}^-(g),\]

where $\hat{\sigma}$ is an estimate of $p(y=+1\mid\bm{x})$ and

\vspace{0.8em}
\resizebox{\linewidth}{!}{
\begin{minipage}{\linewidth}
\begin{align*}
  \bar{R}_{y=-1, \eta, \hat{\sigma}}^-(g)
  &= \mathbb{E}_{\bm{x}\sim p(\bm{x})}[
      \mathds{1}_{\hat{\sigma}(\bm{x})\le\eta}
      \thinspace\ell(-g(\bm{x}))
      (1 - \hat{\sigma}(\bm{x}))]\\
  &\enspace + \pi\thinspace
  \mathbb{E}_{\bm{x}\sim p\sub{P}(\bm{x})}\left[
        \mathds{1}_{\hat{\sigma}(\bm{x})>\eta}
      \thinspace\ell(-g(\bm{x}))
      \frac{1-\hat{\sigma}(\bm{x})}{\hat{\sigma}(\bm{x})}\right].
\end{align*}
\end{minipage}}
\vspace{0.2em}

PUbN\textbackslash N can be viewed as a variant of the traditional
two-step approach in PU learning
which first identifies possible N data in U data and
then perform ordinary PN classification to distinguish P data from
the identified N data.
However, being based on state-of-the-art nnPU learning,
our method is more promising than other similar algorithms.
Moreover, by explicitly considering the posterior
$p(y=+1\mid\bm{x})$, we attempt to correct the bias induced
by the fact of only taking into account confident negative samples.
The benefit of using an unbiased risk estimator is that
the resulting algorithm is always statistically
consistent, i.e., the estimation error converges in probability
to zero as the number of samples grows to infinity.

\section{Experiments} \label{sec: exps}

In this section, we experimentally investigate the proposed method
and compare its performance against several baseline methods.

\subsection{Basic Setup}

We focus on training neural networks with stochastic optimization.
For simplicity, in an experiment, $\hat{\sigma}$ and $g$ always
use the same model and are trained for the same number of epochs.
All models are learned using AMSGrad \citep{j.2018on} as the
optimizer and the logistic loss as the surrogate loss
unless otherwise specified.
In all the experiments, an additional validation set,
equally composed of P, U and bN data,
is sampled for both hyperparameter tuning and choosing the model
parameters with the lowest validation loss among
those obtained after every epoch.
Regarding the computation of the validation loss,
we use the PU risk estimator
\eqref{eq: pu risk estimator} with the sigmoid loss for $g$
and an empirical approximation of
$
  \mathbb{E}_{\bm{x}\sim p(\bm{x})}
  [|\hat{\sigma}(\bm{x})-\sigma(\bm{x})|^2] 
  - \mathbb{E}_{\bm{x}\sim p(\bm{x})}[\sigma(\bm{x})^2]
$
for $\hat{\sigma}$
(see Supplementary Material \ref*{app: val sig}).

\subsection{Effectiveness of the Algorithm} \label{subsec: effectiveness}

We assess the performance of the proposed method on three benchmark
datasets: MNIST, CIFAR-10 and 20 Newsgroups.
Experimental details are given in Supplementary Material \ref*{app: deexp}.
To recapitulate, for the three datasets we respectively use
a 4-layer ConvNet, PreAct ResNet-18 \citep{he2016identity}
and a 3-layer fully connected neural network.
On 20 Newsgroups text features are generated thanks to
the use of ELMo word embedding \citep{Peters:2018}.
Since all the three datasets are originally
designed for multiclass classification, we group different
categories together to form a binary classification problem.

\textbf{Baselines.}\enspace
When $\mathcal{X}\sub{bN}$ is given, two baseline methods are considered.
The first one is nnPNU adapted from \eqref{eq: pnu risk estimator}.
In the second method, named as PU$\rightarrow$PN, we train
two binary classifiers: one is learned with nnPU
while we regard $s$ as the class label,
and the other is learned from $\mathcal{X}\sub{P}$ and $\mathcal{X}\sub{bN}$
to separate P samples from bN samples.
A sample is classified in the P class only if it is so classified
by the two classifiers.
When $\mathcal{X}\sub{bN}$ is not available, nnPU is compared
with the proposed PUbN\textbackslash N.

\textbf{Sampling bN Data.}\enspace
To sample $\mathcal{X}\sub{bN}$, we suppose that the bias of
N data is caused by a latent prior probability change
\citep{sugiyama2007mixture, hu2018does} in the N class.
Let $z\in\mathcal{Z} = \{1, \ldots, S\}$ be some latent variable
which we call a latent category, where $S$ is a constant.
It is assumed

\begin{align*}
  p(\bm{x}\mid z, y=-1) &= p(\bm{x}\mid z, y=-1, s=+1),\\
  p(z\mid y=-1) &\neq p(z\mid y=-1, s=+1).
\end{align*}

\begin{table*}[t]
  \caption{Mean and standard deviation of misclassification rates
    over 10 trials for MNIST, CIFAR-10 and 20 Newsgroups
    under different choices of P class and bN data sampling
    strategies.
    For a same learning task, different methods
    are compared using the same 10 random samplings.
    Underlines denote that with the use of bN data
    the method leads to an improvement of performance according to
    the 5\% t-test.
    Boldface indicates the best method in each task.\\[0.2em]
    $^\dagger$
    Biased N data uniformly sampled from the indicated latent categories.\\
    $^\star$
    Probabilities that a sample of $\mathcal{X}\sub{bN}$ belongs to
    the latent categories [1, 3, 5, 7, 9] /
    [bird, cat, deer, dog, frog, horse] / [sci., soc., talk.]
    are [0.03, 0.15, 0.3, 0.02, 0.5] /
    [0.1, 0.02, 0.2, 0.08, 0.2, 0.4] / [0.1, 0.5, 0.4].%
  }
  \label{tab: effectiveness}
  \begin{center}
    \begin{tabular}{llm{2.7cm}lrrr}
      \toprule
      Dataset & P & biased N & $\rho$
      & nnPU/nnPNU
      & PUbN(\textbackslash N) & PU$\rightarrow$PN \\
      \midrule
      \multirow{3}{*}{MNIST}
      & \multirow{3}{*}{0, 2, 4, 6, 8}
      & Not given & NA
      & $5.76 \pm 1.04$ & $\mathbf{4.64 \pm 0.62}$ & NA \\
      && 1, 3, 5 $^\dagger$ & 0.3
      & $5.33 \pm 0.97$ & \underline{$4.05 \pm 0.27$}
      & \underline{$\mathbf{4.00 \pm 0.30}$} \\
      && 9 $>$ 5 $>$ others $^\star$ & 0.2
      & $4.60 \pm 0.65$ & \underline{$3.91 \pm 0.66$}
      & \underline{$\mathbf{3.77 \pm 0.31}$} \\[0.2em]
      \cmidrule(l{2pt}r{2pt}){1-7} \\[-0.8em]
      \multirow{3}{*}[-0.6em]{CIFAR-10}
      & \multirow{3}{2.5cm}[-0.6em]{Airplane, automobile, ship, truck}
      & Not given & NA
      & $12.02 \pm 0.65$ & $\mathbf{10.70 \pm 0.57}$ & NA \\
      && Cat, dog, horse $^\dagger$ & 0.3
      & $10.25 \pm 0.38$
      & \underline{$\mathbf{9.71 \pm 0.51}$} & $10.37 \pm 0.65$ \\[0.2em]
      && \shortstack[l]{Horse $>$ deer \\= frog $>$ others} $^\star$ & 0.25
      & \underline{$9.98 \pm 0.53$}
      & \underline{$\mathbf{9.92 \pm 0.42}$}
      & \underline{$10.17 \pm 0.35$} \\[0.2em]
      \cmidrule(l{2pt}r{2pt}){1-7} \\[-0.8em]
      \multirow{3}{*}{CIFAR-10}
      & \multirow{3}{2.5cm}{Cat, deer, dog, horse}
      & Not given & NA
      & $23.78 \pm 1.04$ & $\mathbf{21.13 \pm 0.90}$ & NA \\
      && Bird, frog $^\dagger$ & 0.2
      & $22.00 \pm 0.53$
      & \underline{$\mathbf{18.83 \pm 0.71}$}
      & \underline{$19.88 \pm 0.62$} \\
      && Car, truck $^\dagger$ & 0.2
      & $22.00 \pm 0.74$
      & \underline{$\mathbf{20.19 \pm 1.06}$} & $21.83 \pm 1.36$ \\[0.2em]
      \cmidrule(l{2pt}r{2pt}){1-7} \\[-0.8em]
      \multirow{4}{*}{20 Newsgroups}
      & \multirow{4}{2.5cm}{alt., comp., misc., rec.}
      & Not given & NA
      & $14.67 \pm 0.87$ & $\mathbf{13.30 \pm 0.53}$ & NA \\
      && sci.$^\dagger$ & 0.21
      & $14.69 \pm 0.46$
      & $\mathbf{13.10 \pm 0.90}$
      & $13.58 \pm 0.97$ \\
      && talk.$^\dagger$ & 0.17
      & $14.38 \pm 0.74$
      & \underline{$\mathbf{12.61 \pm 0.75}$}
      & $13.76 \pm 0.66$ \\
      && soc. $>$ talk. $>$ sci.$^\star$ & 0.1
      & $14.41 \pm 0.76$
      & \underline{$\mathbf{12.18 \pm 0.59}$}
      & $12.92 \pm 0.51$ \\
      \bottomrule
    \end{tabular}
  \end{center}
\end{table*}

In the experiments, the latent categories are the original
class labels of the datasets.
Concrete definitions of $\mathcal{X}\sub{bN}$
with experimental results are summarized
in \autoref{tab: effectiveness}.

\textbf{Results.}\enspace
Overall, our proposed method consistently achieves the best or
comparable performance in all the scenarios,
including those of standard PU learning.
Additionally, using bN data can effectively help improving
the classification performance.
However, the choice of algorithm is essential.
Both nnPNU and the naive PU$\rightarrow$PN are able to leverage
bN data to enhance classification accuracy in
only relatively few tasks.
In the contrast,
the proposed PUbN successfully reduce the misclassification
error most of the time.

Clearly, the performance gain that we can benefit from
the availability of bN data is case-dependent.
On CIFAR-10, the greatest improvement is achieved when we regard
mammals (i.e. cat, deer, dog and horse)
as P class and drawn samples from latent categories
bird and frog as labeled negative data.
This is not surprising because birds and frogs are more
similar to mammals than vehicles, which makes the
classification harder specifically for samples from
these two latent categories.
By explicitly labeling these samples as N data, we 
allow the classifier to make better predictions for these
difficult samples.

\subsection{Illustration on How the Presence of bN Data Help}

Through experiments we have demonstrated that the
presence of bN data effectively helps learning a better classifier.
Here we would like to provide some intuition for the
reason behind this.
Let us consider the MNIST learning task where $\mathcal{X}\sub{bN}$
is uniformly sampled from the latent categories 1, 3 and 5.
We project the representations learned by the classifier
(i.e., the activation values of the last hidden layer of the neural network)
into a 2D plane using PCA for both nnPU and PUbN algorithms,
as shown in \autoref{fig: PCA}.

\begin{figure}[t]
  \centering
  \includegraphics[width=\linewidth]{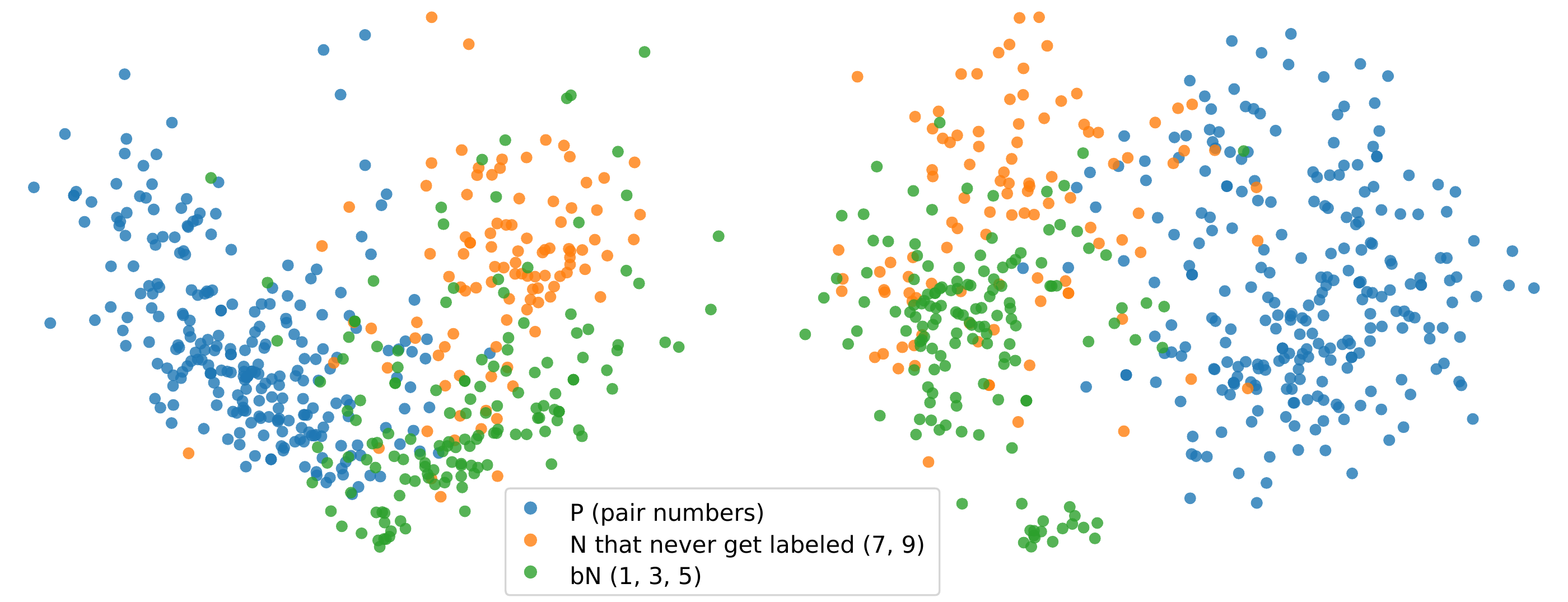}
  \begin{subfigure}{0.48\linewidth}
    \vspace{-0.8em}
    \caption{nnPU}
  \end{subfigure}
  \hfill
  \begin{subfigure}{0.5\linewidth}
    \vspace{-0.8em}
    \caption{PUbN}
  \end{subfigure}
  \vspace{-1em}
  \caption{
    PCA embeddings of the representations learned by the nnPU and
    PUbN classifiers for 500 samples from the test set
    in the MNIST learning task where $\mathcal{X}\sub{bn}$ is
    uniformly sampled from latent categories 1, 3 and 5.
  }
  \label{fig: PCA}
\end{figure}

For both nnPU and PUbN classifiers, the first two 
principal components account around 90\% of variance.
We can therefore presume that the figure depicts
fairly well the learned representations.
Thanks to the use of bN data, in the high-level feature space
1, 3, 5 and P data are further pushed away when we employ
the proposed PUbN learning algorithm,
and we are always able to separate 7, 9 from P to some extent.
This explains the better performance which is achieved by PUbN learning
and the benefit of incorporating bN data into the learning process.

\begin{figure*}[t]
  \centering
  \begin{subfigure}[b]{\linewidth}
    \centering
    \includegraphics[width=0.33\linewidth]{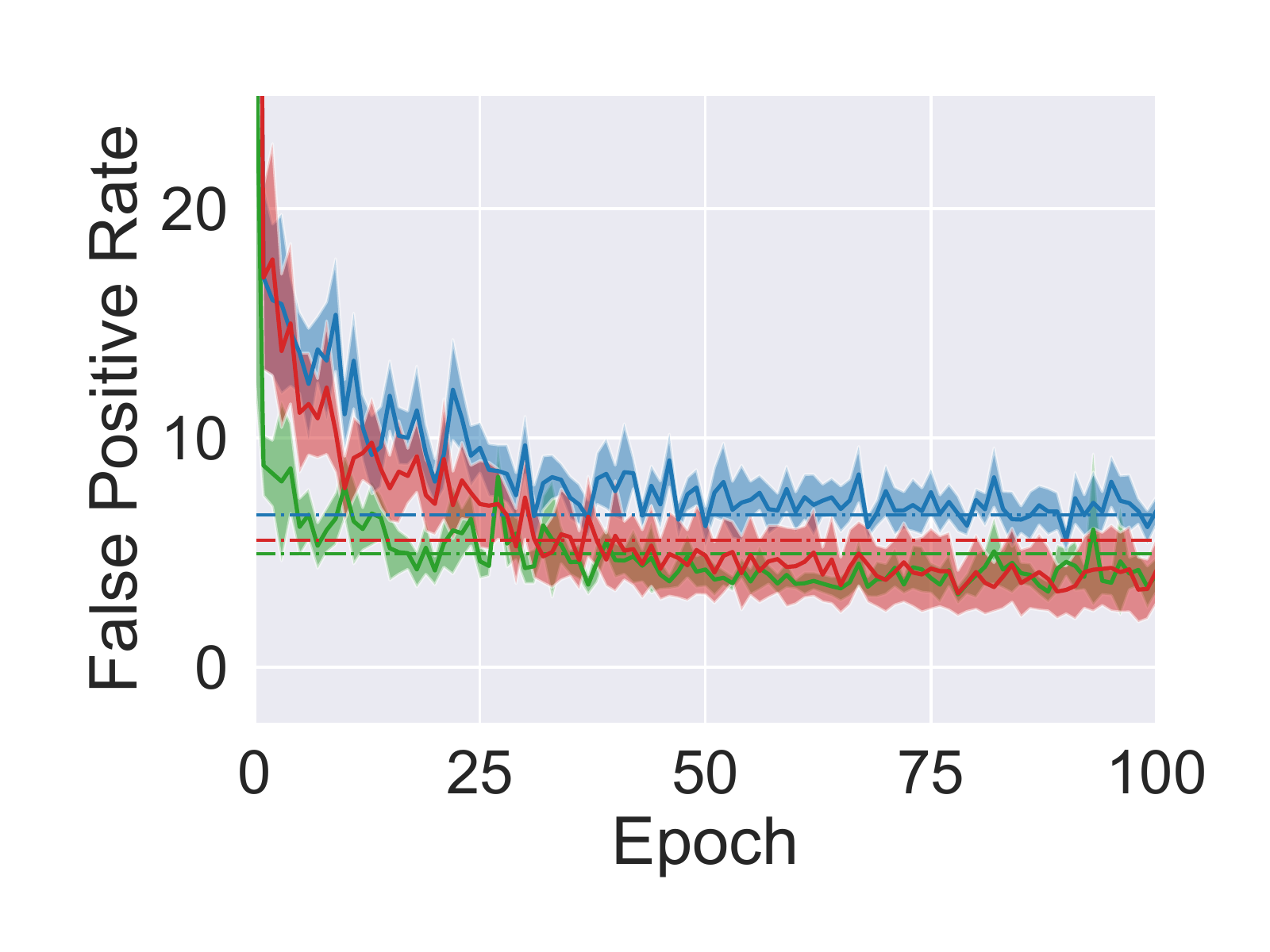}
    \includegraphics[width=0.33\linewidth]{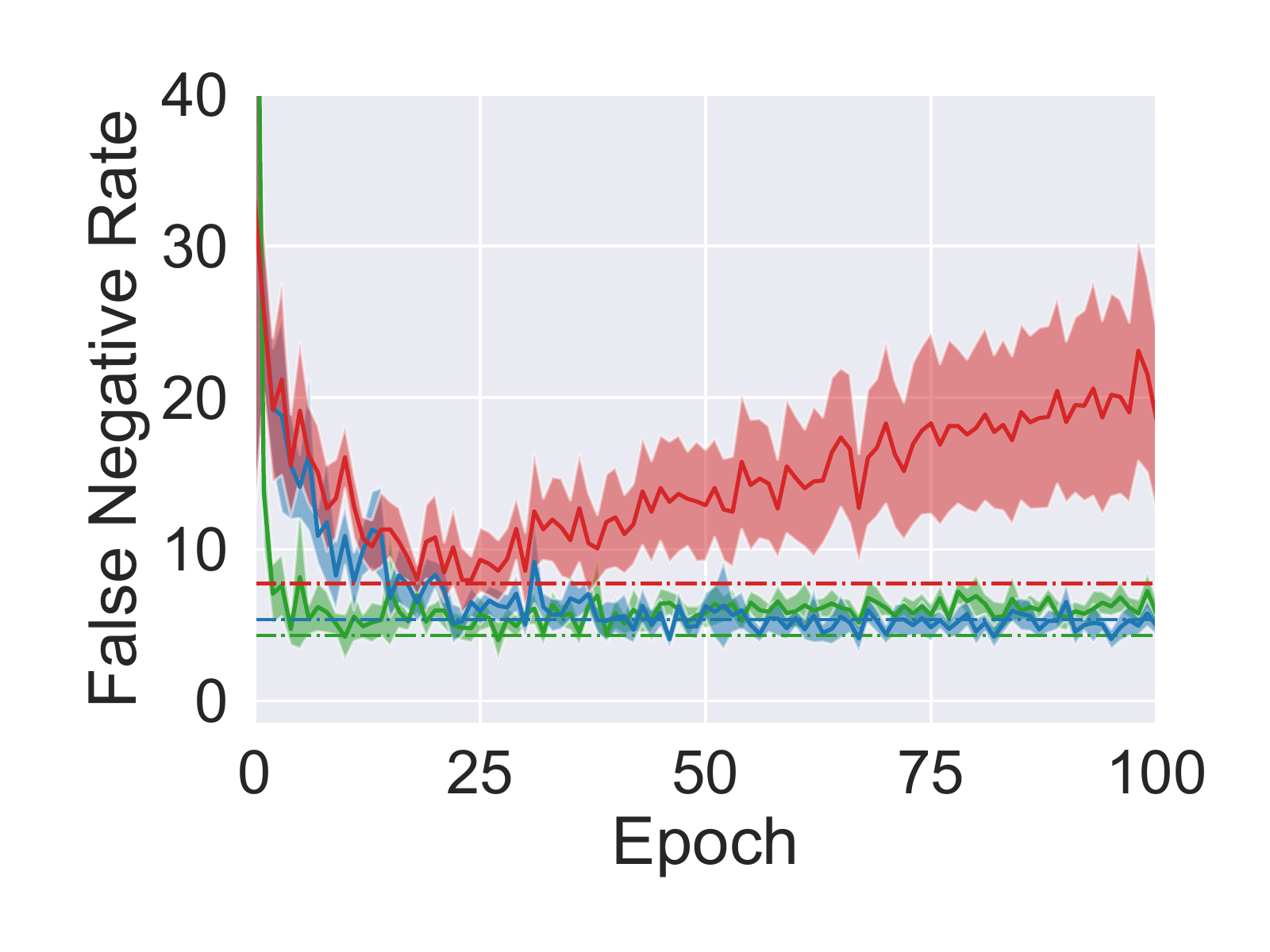}
    \includegraphics[width=0.33\linewidth]{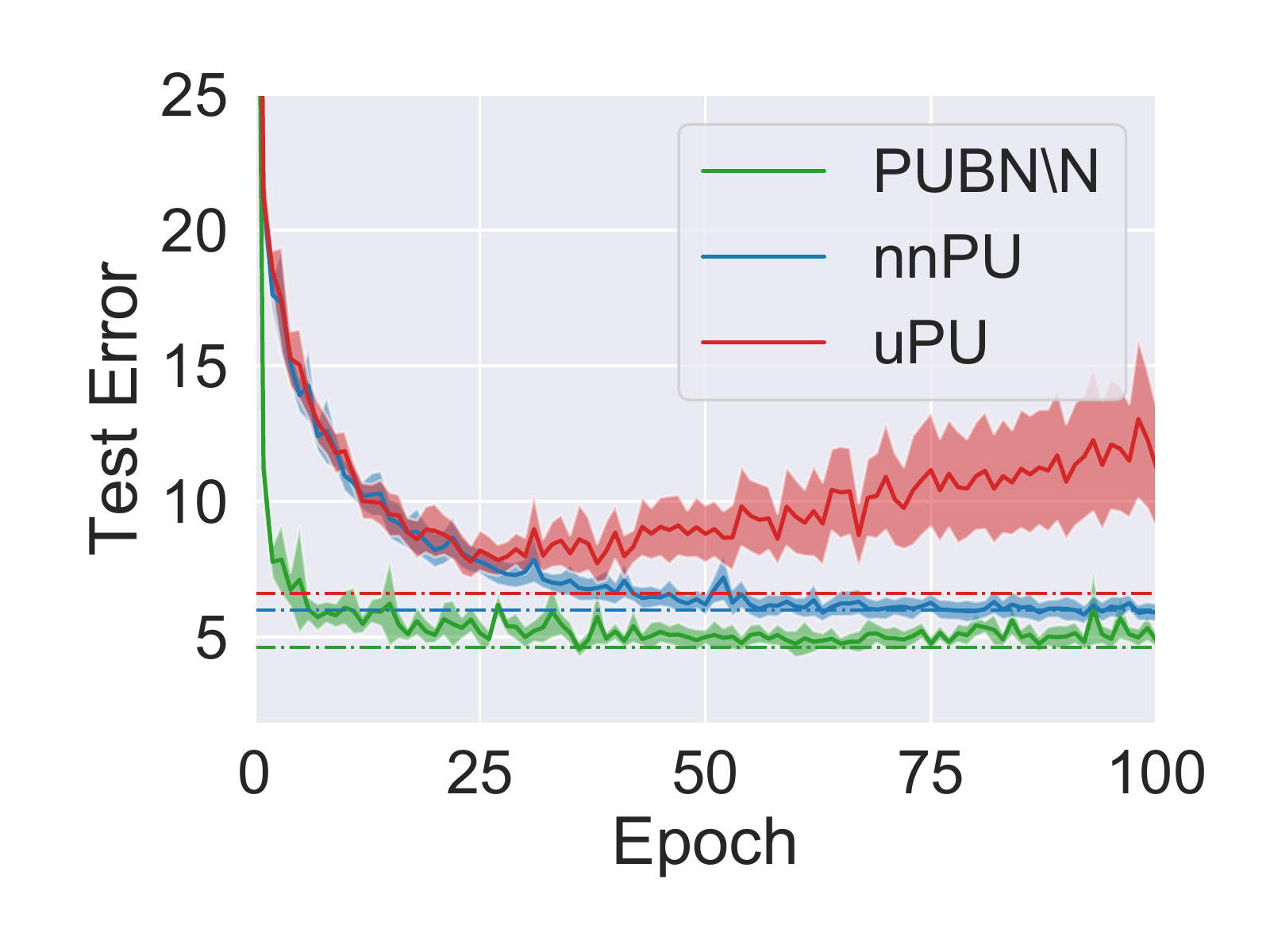}
    \vspace{-2em}
    \caption{MNIST}
  \end{subfigure}
  \begin{subfigure}[b]{\linewidth}
    \centering
    \includegraphics[width=0.33\linewidth]{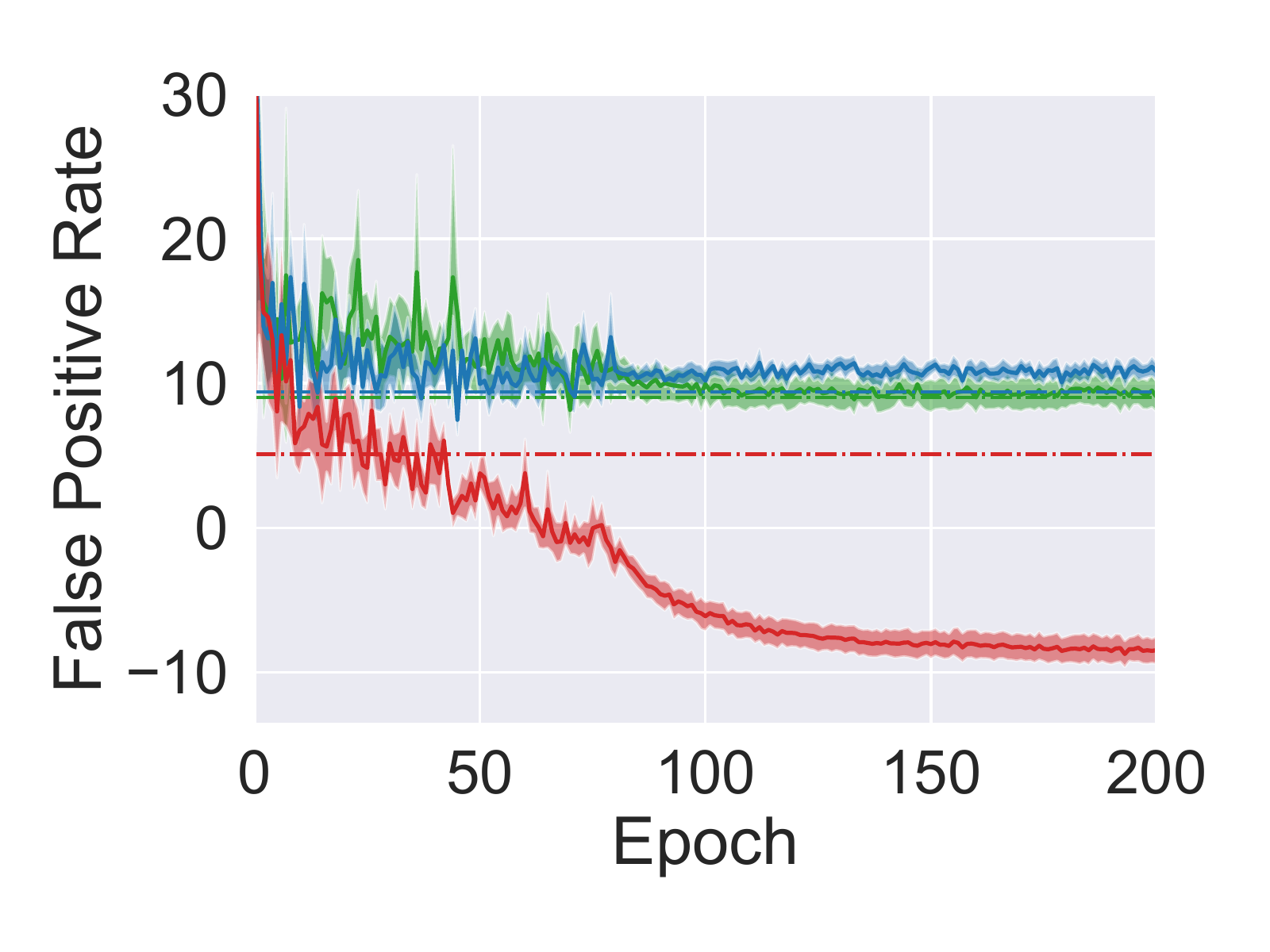}
    \includegraphics[width=0.33\linewidth]{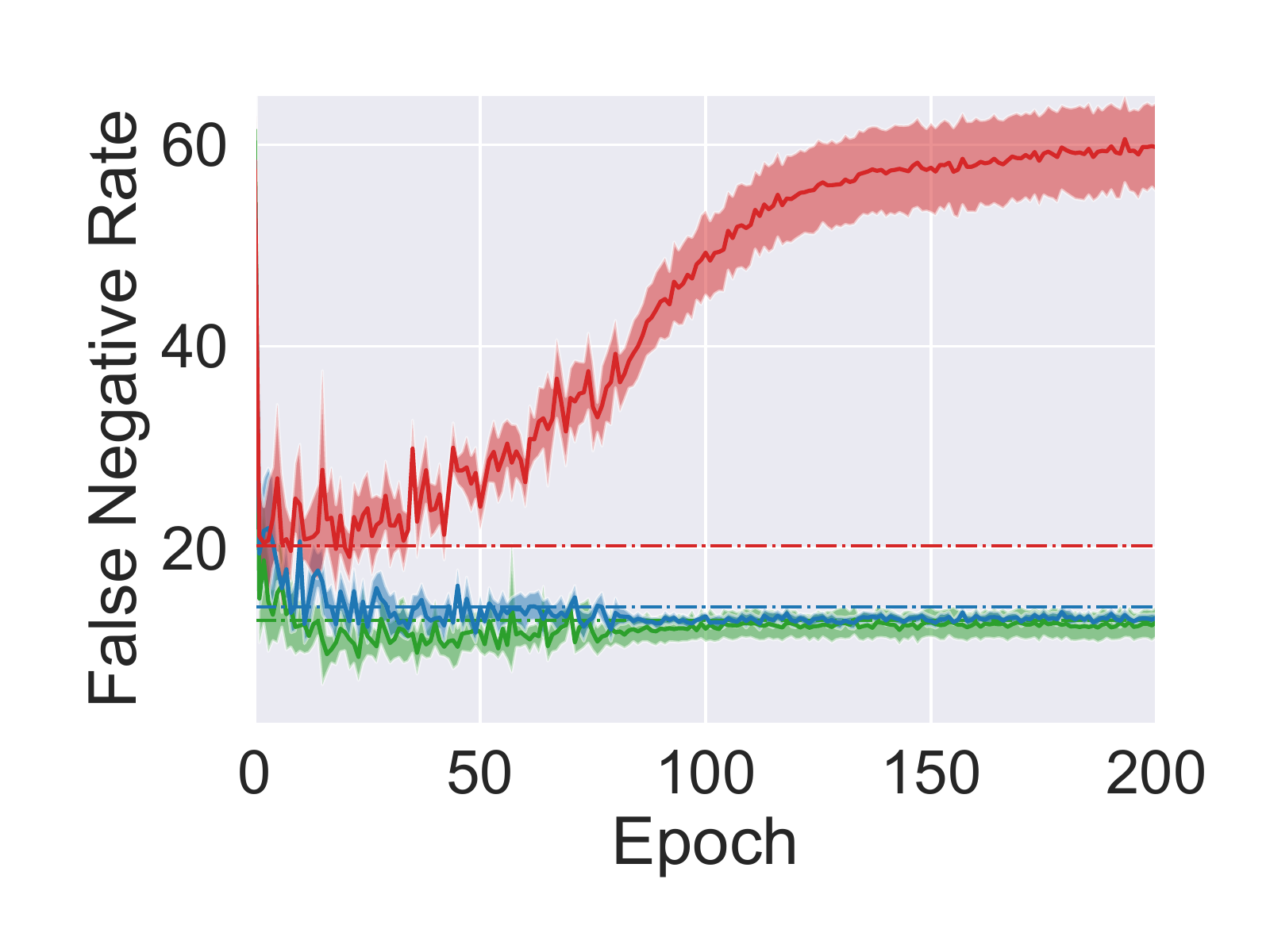}
    \includegraphics[width=0.33\linewidth]{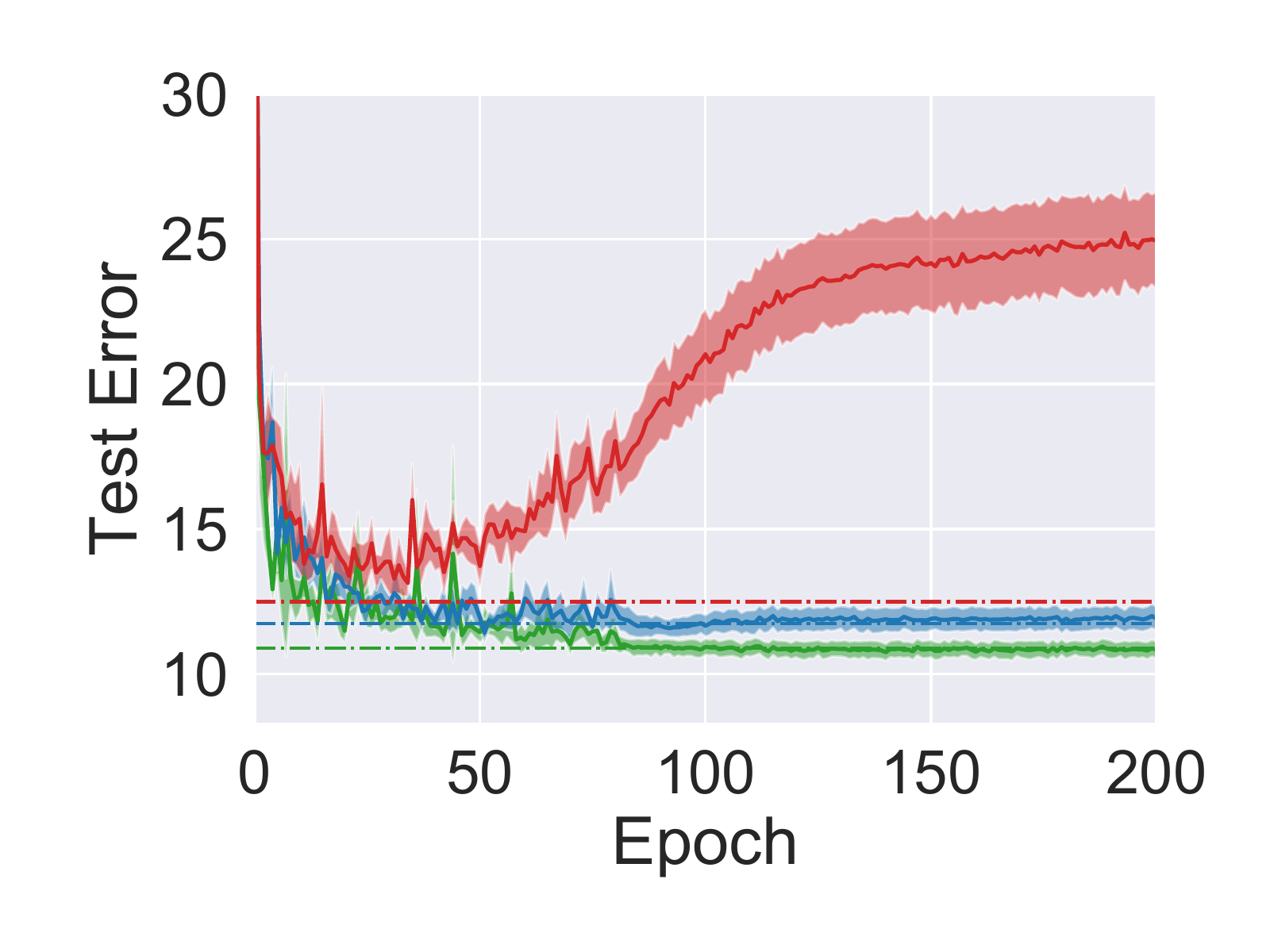}
    \vspace{-2em}
    \caption{CIFAR-10, vehicles as P class}
  \end{subfigure}
  \caption{
    Comparison of uPU, nnPU and PUbN\textbackslash N over two
    of the four PU learning tasks.
    For each task, means and standard deviations are computed
    based on the same 10 random samplings.
    Dashed lines indicate the corresponding values of the final
    classifiers (recall that at the end we select the model
    with the lowest validation loss out of all epochs).
  }
  \label{fig: PU}
\vspace*{-0.2em}
\end{figure*}

\subsection{Why Does PUbN\textbackslash N Outperform nnPU ?}
\label{subsec: why}

Our method, specifically designed for PUbN learning, naturally
outperforms other baseline methods in this problem.
Nonetheless, \autoref{tab: effectiveness} equally shows that
the proposed method when applied
to PU learning, achieves significantly better performance than
the state-of-the-art nnPU algorithm.
Here we numerically investigate the reason behind this phenomenon
with help of the first two PU tasks of the table.

Besides nnPU and PUbN\textbackslash N, we compare with unbiased
PU (uPU) learning \eqref{eq: pu risk estimator}.
Both uPU and nnPU are learned with the sigmoid loss,
learning rate $10^{-3}$ for MNIST and
initial learning rate $10^{-4}$ for CIFAR-10,
as uPU learning is unstable with the logistic loss.
The other parts of the experiments remain unchanged.
On the test sets we compute the false positive rates,
false negative rates and misclassification
errors for the three methods and plot them in \autoref{fig: PU}.
We first notice that PUbN\textbackslash N
still outperforms nnPU trained with the sigmoid loss.
In fact, the final performance of the nnPU classifier
does not change much when we replace the logistic
loss with the sigmoid loss.

In \cite{kiryo2017positive}, the authors observed that uPU overfits
training data with the risk going to negative.
In other words, a large portion of U samples
are classified as N.
This is confirmed in our experiments by an increase of false negative
rate and decrease of false positive rate.
nnPU remedies the problem by introducing the non-negative risk estimator
\eqref{eq: nnpu risk estimator}.
While the non-negative correction successfully prevents false negative
rate from going up, it also causes more N samples to be classified as
P compared to uPU\@.
However, since the gain in terms of false negative rate is enormous,
at the end nnPU achieves a lower misclassification error.
By further identifying possible N samples after nnPU learning,
we expect that our algorithm can yield lower false positive rate
than nnPU without misclassifying too many P samples as N as in the case
of uPU.
\autoref{fig: PU} suggests that this is effectively the case.
In particular, we observe that on MNIST, our method achieves the
same false positive rate as uPU whereas its false negative
rate is comparable to nnPU.

\section{Conclusion}

This paper studies the PUbN classification problem, where a
binary classifier is trained on P, U and bN data.
The proposed method is a two-step approach inspired from
both PU learning and importance weighting.
The key idea is to attribute appropriate weights to each
example for evaluation of the classification
risk using the three sets of data.
We theoretically established an estimation error bound
for the proposed risk estimator and
experimentally
showed that our approach successfully
leveraged bN data to improve the classification performance
on several real-world datasets.
A variant of our algorithm was able to achieve
state-of-the-art results in PU learning.

\section*{Acknowledgements}
MS was supported by JST CREST Grant Number JPMJCR1403.

\bibliography{PUbN}
\bibliographystyle{icml2019}

\newpage

\onecolumn
\icmltitle{
  \vspace{-0.5em}Classification from Positive, Unlabeled and Biased Negative Data:
Supplementary Material}

\appendix

\section*{Appendix}

\section{Proofs}

\subsection{Proof of Theorem 1}

We notice that $(1-\pi-\rho)p(\bm{x}\mid s=-1) = p(\bm{x}, s=-1)$
and that when $h(\bm{x})>\eta$, we have
$p(s=+1\mid\bm{x}) = \sigma(\bm{x}) > 0$, which allows us to write
$p(s=-1\mid\bm{x}) = (p(s=-1\mid\bm{x})/p(s=+1\mid\bm{x}))p(s=+1\mid\bm{x})$.
We can thus decompose $\bar{R}_{s=-1}^-(g)$ as following:

\begin{align*}
  \bar{R}_{s=-1}^-(g)
  &= \int \ell(-g(\bm{x})) p(\bm{x}, s=-1) \,dx \\
  &= \int \mathds{1}_{h(\bm{x})\le\eta}\thinspace
     \ell(-g(\bm{x})) p(\bm{x}, s=-1) \,dx \\
  &\enspace
   + \int \mathds{1}_{h(\bm{x})>\eta}\thinspace
     \ell(-g(\bm{x})) p(\bm{x}, s=-1) \,dx \\
  &= \int \mathds{1}_{h(\bm{x})\le\eta}
     \thinspace\ell(-g(\bm{x}))
     \frac{p(\bm{x}, s=-1)}{p(\bm{x})} p(\bm{x})\,dx \\
  &\enspace
   + \int \mathds{1}_{h(\bm{x})>\eta}
     \thinspace\ell(-g(\bm{x}))
     \frac{p(\bm{x}, s=-1)}{p(\bm{x}, s=+1)} p(\bm{x}, s=+1)\,dx.
\end{align*}

By writing 
$p(\bm{x}, s=-1) = p(s=-1\mid\bm{x})p(\bm{x})
  = (1-\sigma(\bm{x}))p(\bm{x})$
and
$p(\bm{x}, s=+1) = p(s=+1\mid\bm{x})p(\bm{x}) = \sigma(\bm{x})p(\bm{x})$,
we have

\vspace{-6pt}
\begin{align*}
  \bar{R}_{s=-1}^-(g)
  &= \int \mathds{1}_{h(\bm{x})\le\eta}
     \thinspace\ell(-g(\bm{x}))
      (1 - \sigma(\bm{x})) p(\bm{x})\,dx \\
  &\enspace
   + \int \mathds{1}_{h(\bm{x})>\eta}
     \thinspace\ell(-g(\bm{x}))
     \frac{1-\sigma(\bm{x})}{\sigma(\bm{x})} p(\bm{x}, s=+1)\,dx.
\end{align*}

We obtain Equation \eqref{eq: decompose} after replacing
$p(\bm{x}, s=+1)$ by
$\pi p(x\mid y=+1)+\rho p(x\mid y=-1, s=+1)$.

\subsection{Proof of Theorem 2}

For $\hat{\sigma}$ and $\eta$ given, let us define

\[
  R_{\textrm{PUbN},\eta,\hat{\sigma}}(g) = 
  \pi R\sub{P}^+(g) + \rho R\sub{bN}^-(g)
  + \bar{R}_{s=-1, \eta, \hat{\sigma}}^-(g).
\]

The following lemma establishes the uniform deviation bound
from $\hat{R}_{\textrm{PUbN}, \eta, \hat{\sigma}}$
to $R_{\textrm{PUbN}, \eta, \hat{\sigma}}$.

\begin{lemma} \label{lemma: 1}
  Let $\hat{\sigma}: \mathbb{R}^d\rightarrow[0, 1]$ be a fixed
  function independent of data used to compute
  $\hat{R}_{\emph{PUbN}, \eta, \hat{\sigma}}$ and $\eta\in(0,1]$.
  For any $\delta>0$, with probability at least $1-\delta$,

\begin{multline*}
  \sup_{g\in\mathcal{G}}
  |\hat{R}_{\emph{PUbN},\eta,\hat{\sigma}}^-(g)
    -R_{\emph{PUbN},\eta,\hat{\sigma}}(g)|\\
  \enspace\le
  2 L_\ell \mathfrak{R}_{n\sub{\emph{U}}, p}(\mathcal{G})
  + \frac{2\pi L_\ell}{\eta}
    \mathfrak{R}_{n\sub{\emph{P}}, p\sub{\emph{P}}}(\mathcal{G})
  + \frac{2\rho L_\ell}{\eta}
    \mathfrak{R}_{n\sub{\emph{bN}},p\sub{\emph{bN}}}(\mathcal{G})
  + C_\ell\sqrt{\frac{\ln(6/\delta)}{2n\sub{\emph{U}}}}
  + \frac{\pi C_\ell}{\eta}
    \sqrt{\frac{\ln(6/\delta)}{2n\sub{\emph{P}}}}
  + \frac{\rho C_\ell}{\eta}
    \sqrt{\frac{\ln(6/\delta)}{2n\sub{\emph{bN}}}}.
\end{multline*}
\end{lemma}

\begin{proof}
  
  For ease of notation, let 
  
  \begingroup
  \addtolength{\jot}{0.5em}
  \begin{align*}
    R\sub{P}(g) &=
    \mathbb{E}_{\bm{x}\sim p\sub{P}(\bm{x})}
      \left[
        \ell(g(\bm{x}))
        + \mathds{1}_{\hat{\sigma}(\bm{x})>\eta}
          \thinspace\ell(-g(\bm{x}))
          \frac{1-\hat{\sigma}(\bm{x})}{\hat{\sigma}(\bm{x})}
      \right],\\
    R\sub{bN}(g) &=
    \mathbb{E}_{\bm{x}\sim p\sub{bN}(\bm{x})}
      \left[
        \ell(-g(\bm{x}))
        \left(1 + \mathds{1}_{\hat{\sigma}(\bm{x})>\eta}
          \frac{1-\hat{\sigma}(\bm{x})}{\hat{\sigma}(\bm{x})}\right)
      \right],\\
    R\sub{U}(g) &=
    \mathbb{E}_{\bm{x}\sim p(\bm{x})}
      \left[
        \mathds{1}_{\hat{\sigma}(\bm{x})\le\eta}
        \thinspace\ell(-g(\bm{x}))(1-\hat{\sigma}(\bm{x}))
      \right],\\
    \hat{R}\sub{P}(g) &=
    \frac{1}{n\sub{P}}\sum_{i=1}^{n\sub{P}}
      \left[
        \ell(g(\bm{x}_i\super{P}))
        + \mathds{1}_{\hat{\sigma}(\bm{x}_i\super{P})>\eta}
          \thinspace\ell(-g(\bm{x}_i\super{P}))
          \frac{1-\hat{\sigma}(\bm{x}_i\super{P})}
            {\hat{\sigma}(\bm{x}_i\super{P})}
      \right],\\
    \hat{R}\sub{bN}(g) &=
    \frac{1}{n\sub{bN}}\sum_{i=1}^{n\sub{bN}}
      \left[
        \ell(-g(\bm{x}_i\super{bN}))
        \left(1 + \mathds{1}_{\hat{\sigma}(\bm{x}_i\super{bN})>\eta}
          \frac{1-\hat{\sigma}(\bm{x}_i\super{bN})}
               {\hat{\sigma}(\bm{x}_i\super{bN})}\right)
      \right],\\
    \hat{R}\sub{U}(g) &=
    \frac{1}{n\sub{U}}\sum_{i=1}^{n\sub{U}}
      \left[
        \mathds{1}_{\hat{\sigma}(\bm{x}_i\super{U})\le\eta}
        \thinspace\ell(-g(\bm{x}_i\super{U}))
        (1-\hat{\sigma}(\bm{x}_i\super{U}))
      \right].
  \end{align*}
  \endgroup

  From the sub-additivity of the supremum operator, we have

  \begin{align*}
    &\sup_{g\in\mathcal{G}}
    |\hat{R}_{\textrm{PUbN},\eta,\hat{\sigma}}^-(g)
      -R_{\textrm{PUbN},\eta,\hat{\sigma}}(g)|\\
    &\enspace\le
    \pi\sup_{g\in\mathcal{G}}|\hat{R}\sub{P}(g)-R\sub{P}(g)|
    +\rho\sup_{g\in\mathcal{G}}|\hat{R}\sub{bN}(g)-R\sub{bN}(g)|
    +\sup_{g\in\mathcal{G}}|\hat{R}\sub{U}(g)-R\sub{U}(g)|.
  \end{align*}

  As a consequence, to conclude the proof, it suffices to prove that
  with probability at
  least $1-\delta/3$, the following bounds hold separately:
    
  \vspace{-8pt}
  \begin{align}
    \label{eq: p deviation}
    \sup_{g\in\mathcal{G}}|\hat{R}\sub{P}(g)-R\sub{P}(g)|
    & \le \frac{2L_\ell}{\eta}
      \mathfrak{R}_{n\sub{P}, p\sub{P}}(\mathcal{G})
      + \frac{C_\ell}{\eta}
        \sqrt{\frac{\ln(6/\delta)}{2n\sub{P}}},\\
    \label{eq: bn deviation}
    \sup_{g\in\mathcal{G}}|\hat{R}\sub{bN}(g)-R\sub{bN}(g)|
    & \le \frac{2L_\ell}{\eta}
      \mathfrak{R}_{n\sub{bN},p\sub{bN}}(\mathcal{G})
      + \frac{C_\ell}{\eta}
        \sqrt{\frac{\ln(6/\delta)}{2n\sub{bN}}},\\
    \label{eq: u deviation}
    \sup_{g\in\mathcal{G}}|\hat{R}\sub{U}(g)-R\sub{U}(g)|
    & \le 2L_\ell\mathfrak{R}_{n\sub{U}, p}(\mathcal{G})
      + C_\ell\sqrt{\frac{\ln(6/\delta)}{2n\sub{U}}}.
  \end{align}

  Below we prove \eqref{eq: p deviation}.
  \eqref{eq: bn deviation} and \eqref{eq: u deviation} are
  proven similarly.

  Let
  $\phi_{\bm{x}}: \mathbb{R} \rightarrow \mathbb{R}_+$
  be the function defined by
  $\phi_{\bm{x}}: z \mapsto \ell(z)
    + \mathds{1}_{\hat{\sigma}(\bm{x})>\eta}
      \thinspace\ell(-z)
      ((1-\hat{\sigma}(\bm{x}))/\hat{\sigma}(\bm{x}))$.
  For $\bm{x}\in\mathbb{R}^d, g\in\mathcal{G}$,
  since
  $\ell(g(\bm{x}))\in[0, C_\ell]$,
  $\ell(-g(\bm{x}))\in[0, C_\ell]$ and
  $\mathds{1}_{\hat{\sigma}(\bm{x})>\eta}
  ((1-\hat{\sigma}(\bm{x}))/\hat{\sigma}(\bm{x}))
  \in[0, (1-\eta)/\eta]$,
  we always have
  $\phi_{\bm{x}}(g(\bm{x})) \in [0, C_\ell/\eta]$.
  Following the proof of Theorem 3.1 in \cite{mohri2012foundations},
  it is then straightforward to show that with probability at least
  $1-\delta/3$, it holds that

  \[
    \sup_{g\in\mathcal{G}}|\hat{R}\sub{P}(g)-R\sub{P}(g)|
    \le
    2 \thinspace\mathbb{E}_{\mathcal{X\sub{P}}\sim p\sub{P}^{n_{\mathrm{p}}}}
      \mathbb{E}_{\theta}
      \left[
        \sup_{g\in\mathcal{G}}\frac{1}{n\sub{P}}
        \sum_{i=1}^{n\sub{P}}\theta_i \phi_{\bm{x}_i}(g(\bm{x}_i))
      \right]
    + \frac{C_\ell}{\eta}
      \sqrt{\frac{\ln(6/\delta)}{2n\sub{P}}},
  \]

  where $\theta = \{\theta_1, \ldots, \theta_{n\sub{P}}\}$ and each $\theta_i$
  is a Rademacher variable.

  Also notice that for all $\bm{x}$, $\phi_{\bm{x}}$ is a
  $(L_\ell/\eta)$-Lipschitz function on the interval $[-C_g, C_g]$.
  By using a modified version of Talagrad's concentration lemma
  (specifically, Lemma 26.9 in \cite{Shalev-Shwartz:2014:UML:2621980}),
  we can show that, when the set $\mathcal{X}\sub{P}$ is fixed, we have

  \[
    \mathbb{E}_{\theta}
    \left[
      \sup_{g\in\mathcal{G}}\frac{1}{n\sub{P}}
      \sum_{i=1}^{n\sub{P}}\theta_i \phi_{\bm{x}_i}(g(\bm{x}_i))
    \right]
    \le
    \frac{L_\ell}{\eta}
    \mathbb{E}_{\theta}
    \left[
      \sup_{g\in\mathcal{G}}\frac{1}{n\sub{P}}
      \sum_{i=1}^{n\sub{P}}\theta_i g(\bm{x}_i)
    \right].
  \]

  In particular, as the inequality deals with empirical Rademacher complexity,
  the dependence of $\phi_{\bm{x}}$ on $\bm{x}$
  would not be an issue.
  In fact, with $\bm{x}$ being fixed, the indicator function
  $\mathds{1}_{\hat{\sigma}(\bm{x})>\eta}$
  is nothing but a constant and its discontinuity has
  nothing to do with the Lipschitz continuity of $\phi_{\bm{x}}$.
  We obtain Equation \eqref{eq: p deviation}
  After taking expectation over
  $\mathcal{X}\sub{P}\sim p\sub{P}^{n_{\mathrm{p}}}$.\\[-0.7em]
\end{proof}

However, what we really want to minimize is the true risk
$R(g)$. Therefore, we also need to bound the difference
between $R_{\textrm{PUbN},\eta,\hat{\sigma}}(g)$ and $R(g)$,
or equivalently, the difference between 
$\bar{R}_{s=-1,\eta,\hat{\sigma}}^-(g)$ and $\bar{R}_{s=-1}^-(g)$.

\begin{lemma} \label{lemma: 2}
  Let $\hat{\sigma}: \mathbb{R}^d\rightarrow[0, 1]$,
  $\eta\in(0,1]$, $\zeta=p(\hat{\sigma}\le\eta)$ and
  $\epsilon=\mathbb{E}_{\bm{x}\sim p(\bm{x})}
   [|\hat{\sigma}(\bm{x})-\sigma(\bm{x})|^2]$.
  For all $g\in\mathcal{G}$, it holds that

  \[
    |\bar{R}_{s=-1,\eta,\hat{\sigma}}^-(g) - \bar{R}_{s=-1}^-(g)|
    \le C_\ell\sqrt{\zeta\epsilon}
    + \frac{C_\ell}{\eta}\sqrt{(1-\zeta)\epsilon}.
  \]

\end{lemma}

\begin{proof}
  One one hand, we have
  
  \begin{align*}
    \bar{R}_{s=-1}^-(g)
    &= \underbrace{%
        \int \mathds{1}_{\hat{\sigma}(\bm{x})\le\eta}
        \thinspace\ell(-g(\bm{x}))
        (1 - \sigma(\bm{x})) p(\bm{x})\,dx}_{A_1} \\
    &\enspace
     + \underbrace{%
        \int \mathds{1}_{\hat{\sigma}(\bm{x})>\eta}
        \thinspace\ell(-g(\bm{x}))
        (1 - \sigma(\bm{x})) p(\bm{x})\,dx}_{B_1}.
  \end{align*}

  On the other hand, we can express
  $\bar{R}_{s=-1,\eta,\hat{\sigma}}^-(g)$ as

  \begin{align*}
    \bar{R}_{s=-1,\eta,\hat{\sigma}}^-(g)
    &= \int \mathds{1}_{\hat{\sigma}(\bm{x})\le\eta}
       \thinspace\ell(-g(\bm{x}))
       (1 - \hat{\sigma}(\bm{x})) p(\bm{x})\,dx \\
    &\enspace
     + \int \mathds{1}_{\hat{\sigma}(\bm{x})>\eta}
       \thinspace\ell(-g(\bm{x}))
       \frac{1 - \hat{\sigma}(\bm{x})}{\hat{\sigma}(x)}
       p(\bm{x}, s=+1)\,dx.\\
    &= \underbrace{%
        \int \mathds{1}_{\hat{\sigma}(\bm{x})\le\eta}
        \thinspace\ell(-g(\bm{x}))
        (1 - \hat{\sigma}(\bm{x})) p(\bm{x})\,dx}_{A_2} \\
    &\enspace
     + \underbrace{%
        \int \mathds{1}_{\hat{\sigma}(\bm{x})>\eta}
        \thinspace\ell(-g(\bm{x}))
        (1 - \hat{\sigma}(\bm{x}))
        \frac{\sigma(\bm{x})}{\hat{\sigma}(\bm{x})}
        p(\bm{x})\,dx}_{B_2}.
  \end{align*}

  The last equality follows from
  $p(\bm{x}, s=+1)=\sigma(\bm{x})p(\bm{x})$.
  As 
  $|\bar{R}_{s=-1,\eta,\hat{\sigma}}^-(g) - \bar{R}_{s=-1}^-(g)|
  \le |A_1-A_2| + |B_1-B_2|$,
  it is sufficient to derive bounds for $|A_1-A_2|$ and $|B_1-B_2|$ separately.
  For $|B_1-B_2|$, we write

  \begin{align*}
    |B_1-B_2| &\le 
      \int \mathds{1}_{\hat{\sigma}(\bm{x})>\eta}
      \thinspace\ell(-g(\bm{x}))
      \frac{|\hat{\sigma}(\bm{x})-\sigma(\bm{x})|}{\hat{\sigma}(\bm{x})}
      p(\bm{x})\,dx\\
    &\le 
      \frac{C_\ell}{\eta}
      \int \mathds{1}_{\hat{\sigma}(\bm{x})>\eta}
      |\hat{\sigma}(\bm{x})-\sigma(\bm{x})|
      p(\bm{x})\,dx\\
    &\le 
      \frac{C_\ell}{\eta}
      \left(\int \mathds{1}_{\hat{\sigma}(\bm{x})>\eta}^2 p(\bm{x})\,dx\right)
      ^{\frac{1}{2}}
      \left(\int |\hat{\sigma}(\bm{x})-\sigma(\bm{x})|^2 p(\bm{x})\,dx\right)
      ^{\frac{1}{2}}\\
    &= \frac{C_\ell}{\eta}\sqrt{(1-\zeta)\epsilon}
  \end{align*}

  From the second to the third line we use the Cauchy-Schwarz inequality.
  $|A_1-A_2| \le C_\ell\sqrt{\zeta\epsilon}$ can be proven similarly, which
  concludes the proof.\\[-0.7em]
\end{proof}

Combining lemma \ref{lemma: 1} and lemma \ref{lemma: 2}, we know that
with probability at least $1-\delta$, the following holds:

\begin{align*}
  &\sup_{g\in\mathcal{G}}
  |\hat{R}_{\textrm{PUbN},\eta,\hat{\sigma}}^-(g)-R(g)|\\
  &\enspace\le
  2 L_\ell \mathfrak{R}_{n\sub{U}, p}(\mathcal{G})
  + \frac{2\pi L_\ell}{\eta}
    \mathfrak{R}_{n\sub{P}, p\sub{P}}(\mathcal{G})
  + \frac{2\rho L_\ell}{\eta}
    \mathfrak{R}_{n\sub{bN},p\sub{bN}}(\mathcal{G}) \\
  &\enspace\quad
  + C_\ell\sqrt{\frac{\ln(6/\delta)}{2n\sub{U}}}
  + \frac{\pi C_\ell}{\eta}
    \sqrt{\frac{\ln(6/\delta)}{2n\sub{P}}}
  + \frac{\rho C_\ell}{\eta}
    \sqrt{\frac{\ln(6/\delta)}{2n\sub{bN}}}
  + C_\ell\sqrt{\zeta\epsilon}
  + \frac{C_\ell}{\eta}\sqrt{(1-\zeta)\epsilon}.
\end{align*}

Finally, with probability at least $1-\delta$,

\begingroup
\addtolength{\jot}{0.5em}
\begin{align*}
  &R(\hat{g}_{\textrm{PUbN}, \eta, \hat{\sigma}}) - R(g^*) \\
  &\enspace= (R(\hat{g}_{\textrm{PUbN}, \eta, \hat{\sigma}}) 
    - \hat{R}_{\textrm{PUbN},\eta,\hat{\sigma}}^-
      (\hat{g}_{\textrm{PUbN}, \eta, \hat{\sigma}}))\\
  &\enspace\quad
    + (\hat{R}_{\textrm{PUbN},\eta,\hat{\sigma}}^-
      (\hat{g}_{\textrm{PUbN}, \eta, \hat{\sigma}})
    - \hat{R}_{\textrm{PUbN},\eta,\hat{\sigma}}^-(g^*)) 
    + (\hat{R}_{\textrm{PUbN},\eta,\hat{\sigma}}^-(g^*) - R(g^*)) \\
  &\enspace\le
  \sup_{g\in\mathcal{G}}
  |\hat{R}_{\textrm{PUbN},\eta,\hat{\sigma}}^-(g)-R(g)|
  + 0
  + \sup_{g\in\mathcal{G}}
  |\hat{R}_{\textrm{PUbN},\eta,\hat{\sigma}}^-(g)-R(g)|\\
  &\enspace\le
  4 L_\ell \mathfrak{R}_{n\sub{U}, p}(\mathcal{G})
  + \frac{4\pi L_\ell}{\eta}
    \mathfrak{R}_{n\sub{P}, p\sub{P}}(\mathcal{G})
  + \frac{4\rho L_\ell}{\eta}
    \mathfrak{R}_{n\sub{bN},p\sub{bN}}(\mathcal{G}) \\
  &\enspace\quad
  + 2C_\ell\sqrt{\frac{\ln(6/\delta)}{2n\sub{U}}}
  + \frac{2\pi C_\ell}{\eta}
    \sqrt{\frac{\ln(6/\delta)}{2n\sub{P}}}
  + \frac{2\rho C_\ell}{\eta}
    \sqrt{\frac{\ln(6/\delta)}{2n\sub{bN}}}
  + 2C_\ell\sqrt{\zeta\epsilon}
  + \frac{2C_\ell}{\eta}\sqrt{(1-\zeta)\epsilon}.
\end{align*}
\endgroup

The first inequality uses the definition of
$\hat{g}_{\textrm{PUbN},\eta,\hat{\sigma}}$.

\section{Validation Loss for Estimation of $\sigma$} \label{app: val sig}

In terms of validation we want to choose the model for $\hat{\sigma}$
such that $J_0(\hat{\sigma}) =
\mathbb{E}_{\bm{x}\sim p(\bm{x})}
[|\hat{\sigma}(\bm{x})-\sigma(\bm{x})|^2]$
is minimized.
Since $\sigma(\bm{x})p(\bm{x}) = p(\bm{x}, s=+1)$, we have

\begin{align*}
  J_0(\hat{\sigma})
  &= \int (\hat{\sigma}(\bm{x}) - \sigma(\bm{x}))^2 p(\bm{x}) \,dx\\
  &= \int \hat{\sigma}(\bm{x})^2p(\bm{x}) \,dx
   - 2\int\hat{\sigma}(\bm{x})p(\bm{x}, s=+1)\,dx
   + \int\sigma(\bm{x})^2p(\bm{x})\,dx.
\end{align*}

The last term does not depend on $\hat{\sigma}$ and can be ignored
if we want to identify $\hat{\sigma}$ achieving the smallest
$J(\hat{\sigma})$. We denote by $J(\hat{\sigma})$ the sum
of the first two terms.
The middle term can be further expanded using
  
\[
  \int\hat{\sigma}(\bm{x})p(\bm{x}, s=+1)\,dx
  = \pi\int\hat{\sigma}(\bm{x})p(\bm{x}\mid y=+1)\,dx
  + \rho\int\hat{\sigma}(\bm{x})p(\bm{x}\mid y=-1, s=+1)\,dx.
\]

The validation loss of an estimation $\hat{\sigma}$ is then
defined as

\[
  \hat{J}(\hat{\sigma})
  = \frac{1}{n\sub{U}}\sum_{i=1}^{n\sub{U}}\hat{\sigma}(\bm{x}_i\super{U})^2
  - \frac{2\pi}{n\sub{P}}\sum_{i=1}^{n\sub{P}}\hat{\sigma}(\bm{x}_i\super{P})
  -
  \frac{2\rho}{n\sub{bN}}\sum_{i=1}^{n\sub{bN}}\hat{\sigma}(\bm{x}_i\super{bN}).
\]

It is also possible to minimize this value directly to acquire $\hat{\sigma}$.
In our experiments we decide to learn $\hat{\sigma}$ by nnPU for
a better comparison between different methods.

\section{Detailed Experimental Setting} \label{app: deexp}

\subsection{From Multiclass to Binary Class}

In the experiments we work on multiclass classification
datasets.
Therefore it is necessary to define the P and N classes ourselves.
MNIST is processed in such a way that pair numbers 0, 2, 4, 6, 8
form the P class and impair numbers 1, 3, 5, 7, 9 form the N class.
Accordingly, $\pi=0.49$.
For CIFAR-10, we consider two definitions of the P class.
The first one corresponds to a quite natural task that aims to
distinguish vehicles from animals.
Airplane, automobile, ship and truck are therefore defined to
be the P class while the N class is formed by bird, cat, deer, dog,
frog and horse.
For the sake of diversity, we also study another task
in which we attempt to distinguish the mammals from
the non-mammals.
The P class is  then formed by cat, deer, dog, and horse while
the N class consists of the other six categories.
We have $\pi=0.4$ in the two cases.
As for 20 Newsgroups, alt., comp., misc.\ and rec.\ make up the
P class whereas sci., soc.\ and talk.\ make up the N class.
This gives $\pi=0.56$.

\subsection{Training, Validation and Test Set} \label{app: tvt}

For the three datasets, we use the standard test examples as a
held-out test set.
The test set size is thus of 10000 for MNIST and CIFAR-10,
and 7528 for 20 Newsgroups.
Regarding the training set, we sample 500, 500 and 6000
P, bN and U training examples for MNIST and 20 Newsgroups,
and 1000, 1000 and 10000 P, bN and U training examples
for CIFAR-10.
The validation set is always five times smaller than the
training set.

\subsection{20 Newsgroups Preprocessing}

The original 20 Newsgroups dataset contains raw text data
and needs to be preprocessed into text feature vectors
for classification.
In our experiments we borrow the pre-trained
ELMo word embedding \citep{Peters:2018} from
\href{https://allennlp.org/elmo}{https://allennlp.org/elmo}.
The used 5.5B model was, according to the website,
trained on a dataset of 5.5B tokens
consisting of Wikipedia (1.9B) and all of the monolingual news
crawl data from WMT 2008-2012 (3.6B). 
For each word, we concatenate the features from the
three layers of the ELMo model, and for each document,
as suggested by \citet{andreas2018conc},
we concatenate the average, minimum, and maximum computed
along the word dimension.
This results in a 9216-dimensional feature vector for
a single document.

\subsection{Models and Hyperparameters}

\textbf{Shared}\enspace
The nnPU threshold parameter $\beta$ and the weight decay
are respectively fixed at $0$ and $10^{-4}$.
Other hyperparameters including $\tau\in\{0.5, 0.7, 0.9\}$,
$\gamma\in\{0.1, 0.3, 0.5, 0.7, 0.9\}$ and learning rate
are selected with validation data.

\textbf{MNIST}\enspace
For MNIST, we use a standard ConvNet with ReLU.
This model contains two 5x5 convolutional layers and one
fully-connected layer, with each convolutional layer
followed by a 2x2 max pooling.
The channel sizes are 5-10-40.
The model is trained for 100 epochs with each minibatch 
made up of 10 P, 10 bN
(if available) and 120 U samples. 
The learning rate is selected from the range
$\alpha\in\{10^{-2}, 10^{-3}\}$.

\textbf{CIFAR-10}\enspace
For CIFAR-10, we train PreAct ResNet-18 
\citep{he2016identity} for 200 epochs and
the learning rate is divided by 10 after 80 epochs and 120 epochs.
This is a common practice and similar adjustment can be found
in \cite{he2016identity}. 
The minibatch size is 1/100 of the number of training
samples, and the initial learning rate is chosen from
$\{10^{-2}, 10^{-3}\}$.

\textbf{20 Newsgroups}\enspace
For 20 Newsgroups, with the extracted features, we simply
train a multilayer perceptron with two hidden layers of
300 neurons for 50 epochs.
We use basically the same hyperparameters as for MNIST except
that the learning rate $\alpha$ is selected from
$\{5\cdot 10^{-3}, 10^{-3}, 5\cdot 10^{-4}\}$.

\section{Additional Experiments}

\subsection{Why Does PUbN\textbackslash N Outperform nnPU ?}

Here we complete the results presented in
Section \ref{subsec: why} with the plots on the other
two PU learning tasks (\autoref{fig: PU_comp}).
We recall that we compare between PUbN\textbackslash N,
nnPU and uPU learning, and that
both uPU and nnPU are learned with the sigmoid loss,
learning rate $10^{-3}$ for MNIST
and initial learning rate $10^{-4}$ for CIFAR-10.
The learning rate is $10^{-4}$ for 20 Newsgroups.

\begin{figure}[h!]
  \centering
  \begin{subfigure}[b]{\textwidth}
    \centering
    \includegraphics[width=0.32\linewidth]{MNIST_PU_false_positive_rate}
    \includegraphics[width=0.32\linewidth]{MNIST_PU_false_negative_rate}
    \includegraphics[width=0.32\linewidth]{MNIST_PU_error}
    \vspace{-0.5em}
    \caption{MNIST}
  \end{subfigure}
  \begin{subfigure}[b]{\textwidth}
    \centering
    \includegraphics[width=0.32\linewidth]{CIFAR-10_tPU_false_positive_rate}
    \includegraphics[width=0.32\linewidth]{CIFAR-10_tPU_false_negative_rate}
    \includegraphics[width=0.32\linewidth]{CIFAR-10_tPU_error}
    \vspace{-0.5em}
    \caption{CIFAR-10, vehicles as P class}
  \end{subfigure}
  \begin{subfigure}[b]{\textwidth}
    \centering
    \includegraphics[width=0.32\linewidth]{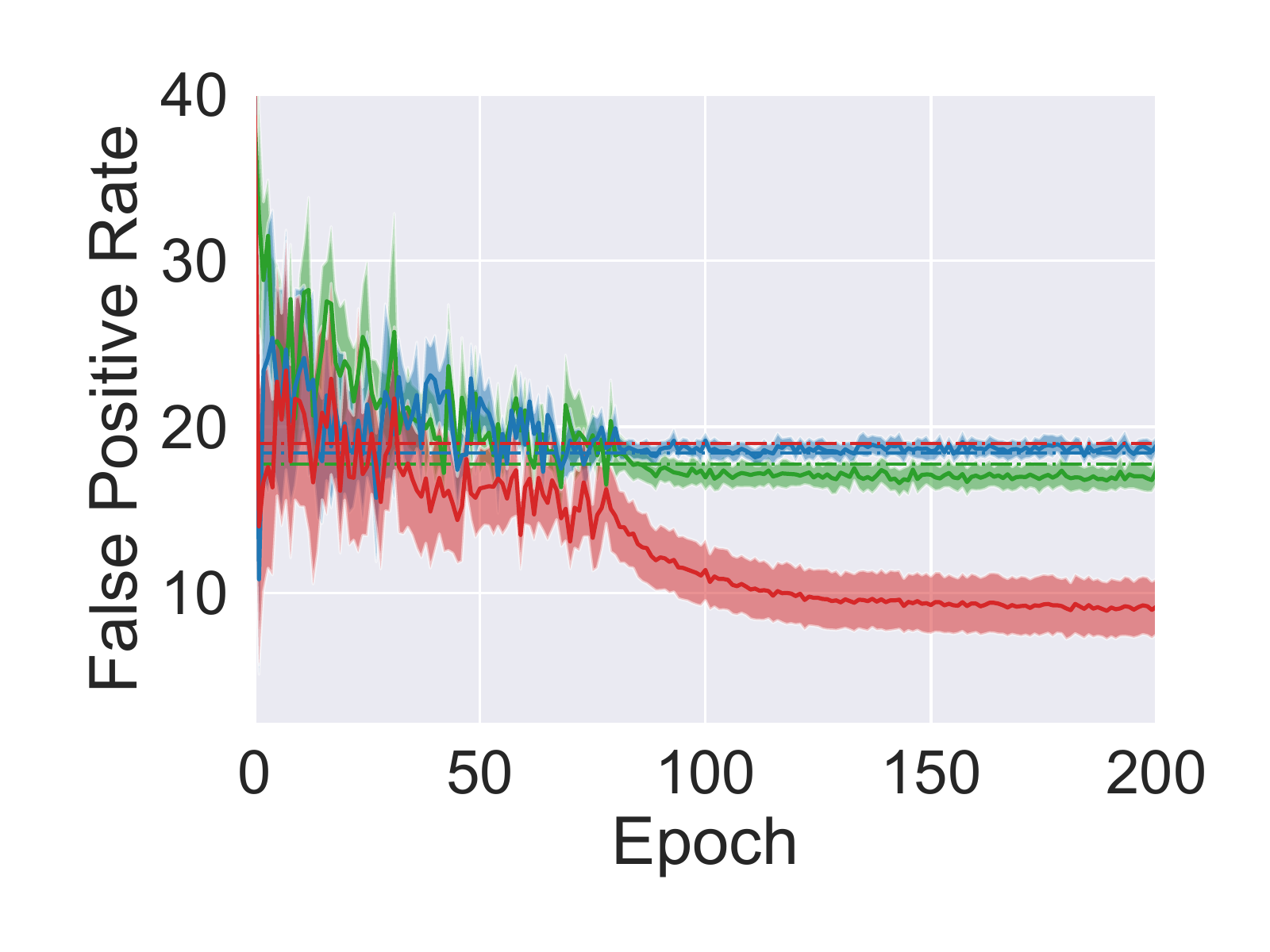}
    \includegraphics[width=0.32\linewidth]{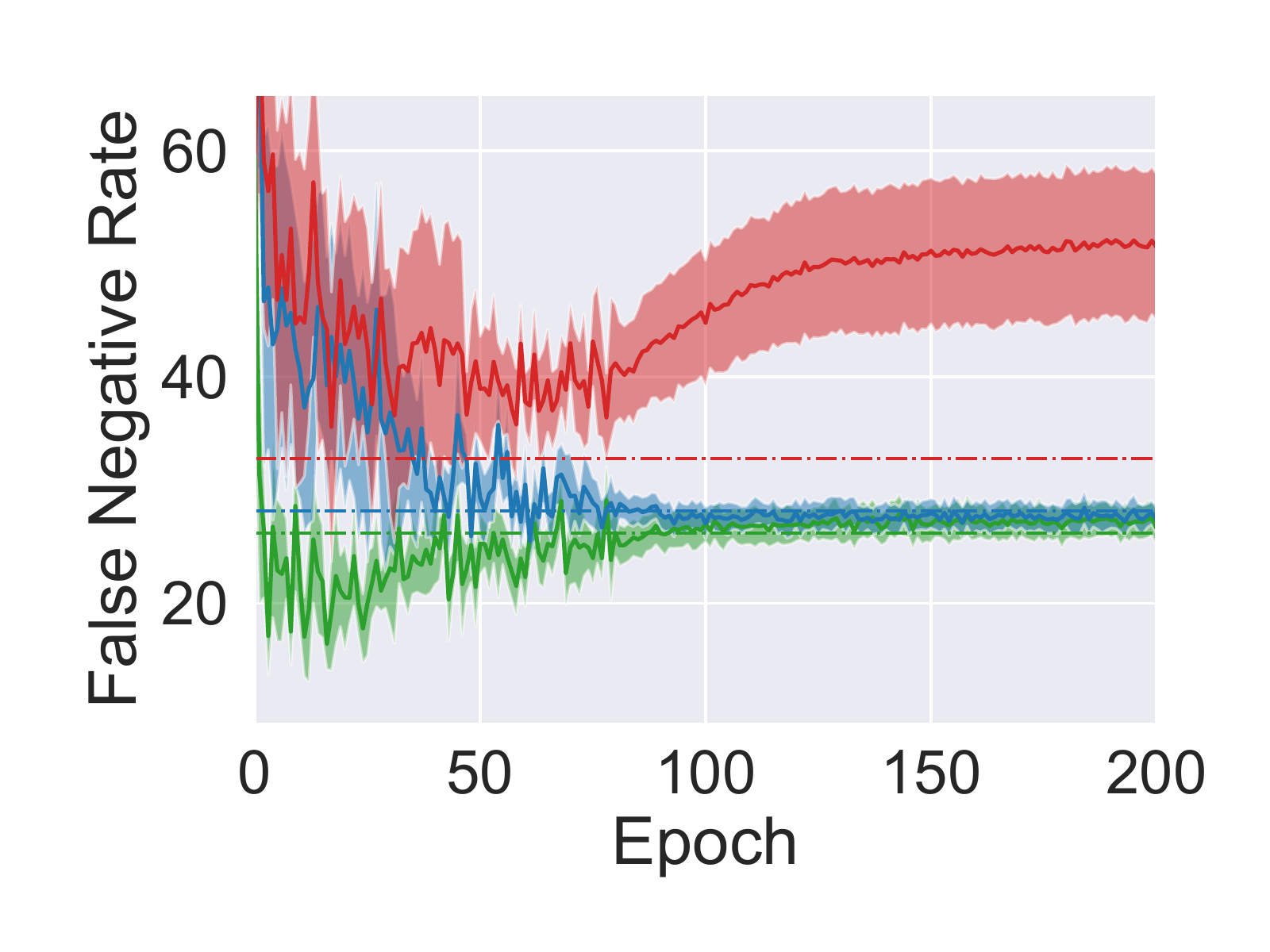}
    \includegraphics[width=0.32\linewidth]{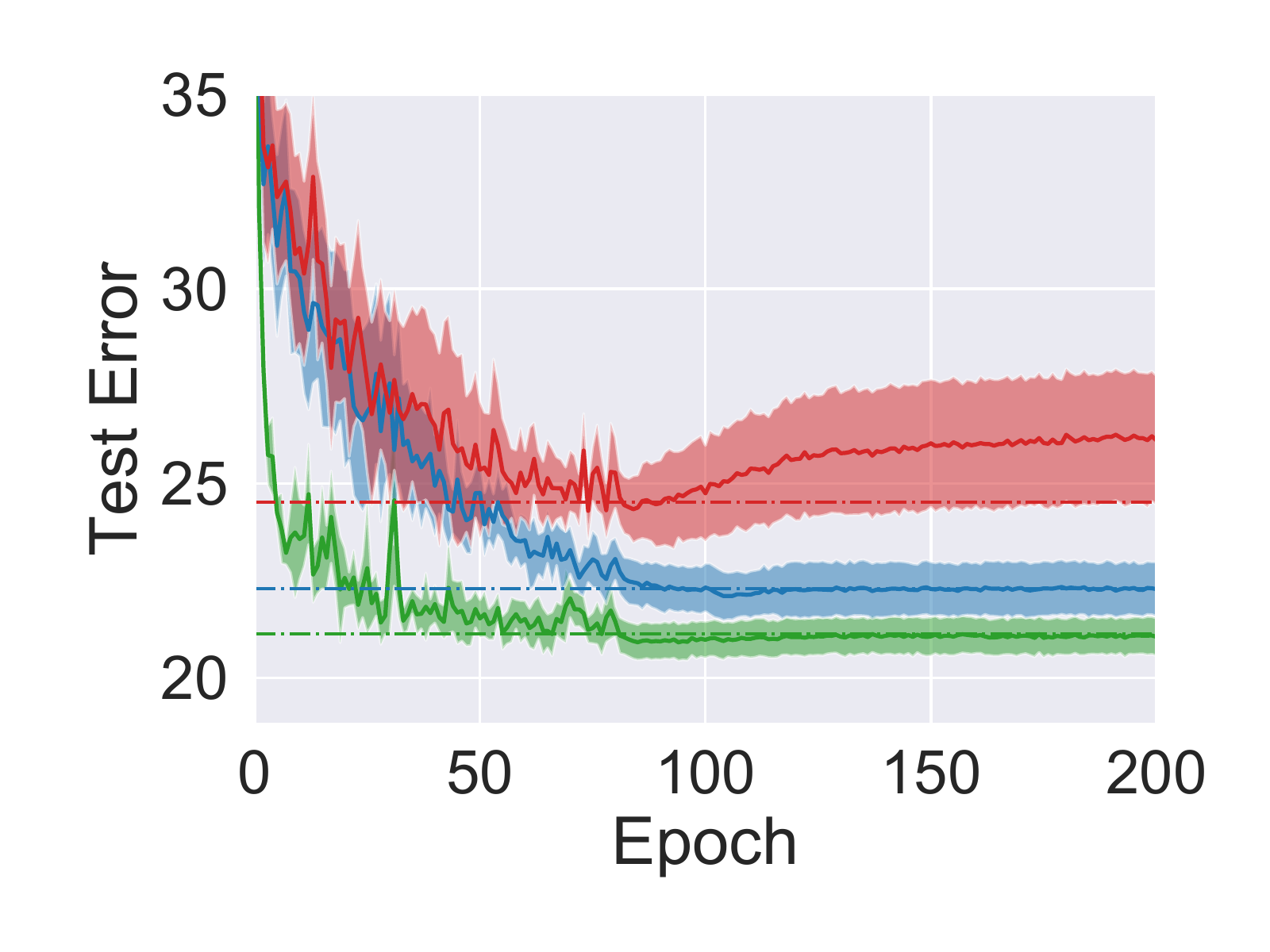}
    \vspace{-0.5em}
    \caption{CIFAR-10, mammals as P class}
  \end{subfigure}
  \begin{subfigure}[b]{\textwidth}
    \centering
    \includegraphics[width=0.32\linewidth]{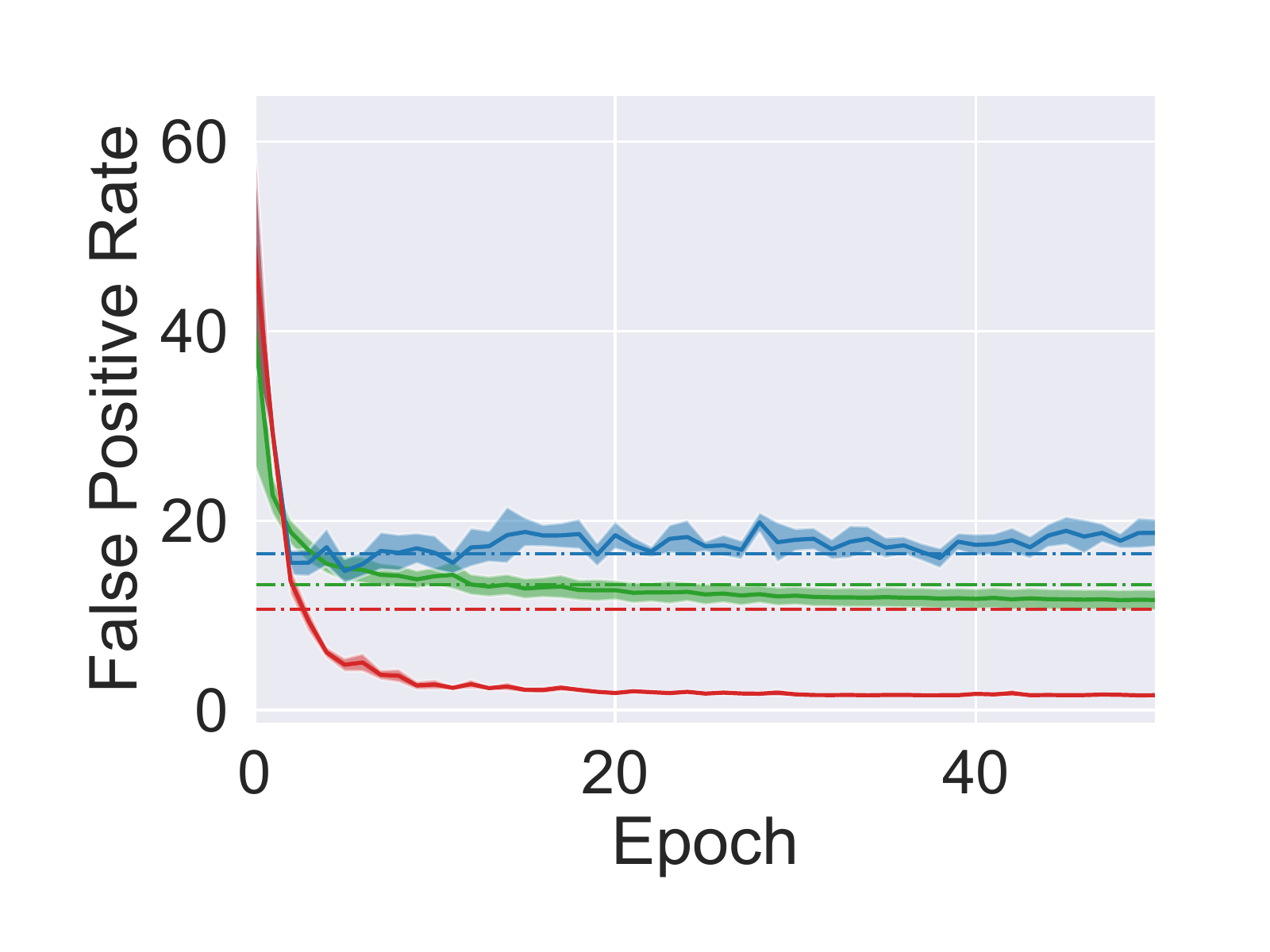}
    \includegraphics[width=0.32\linewidth]{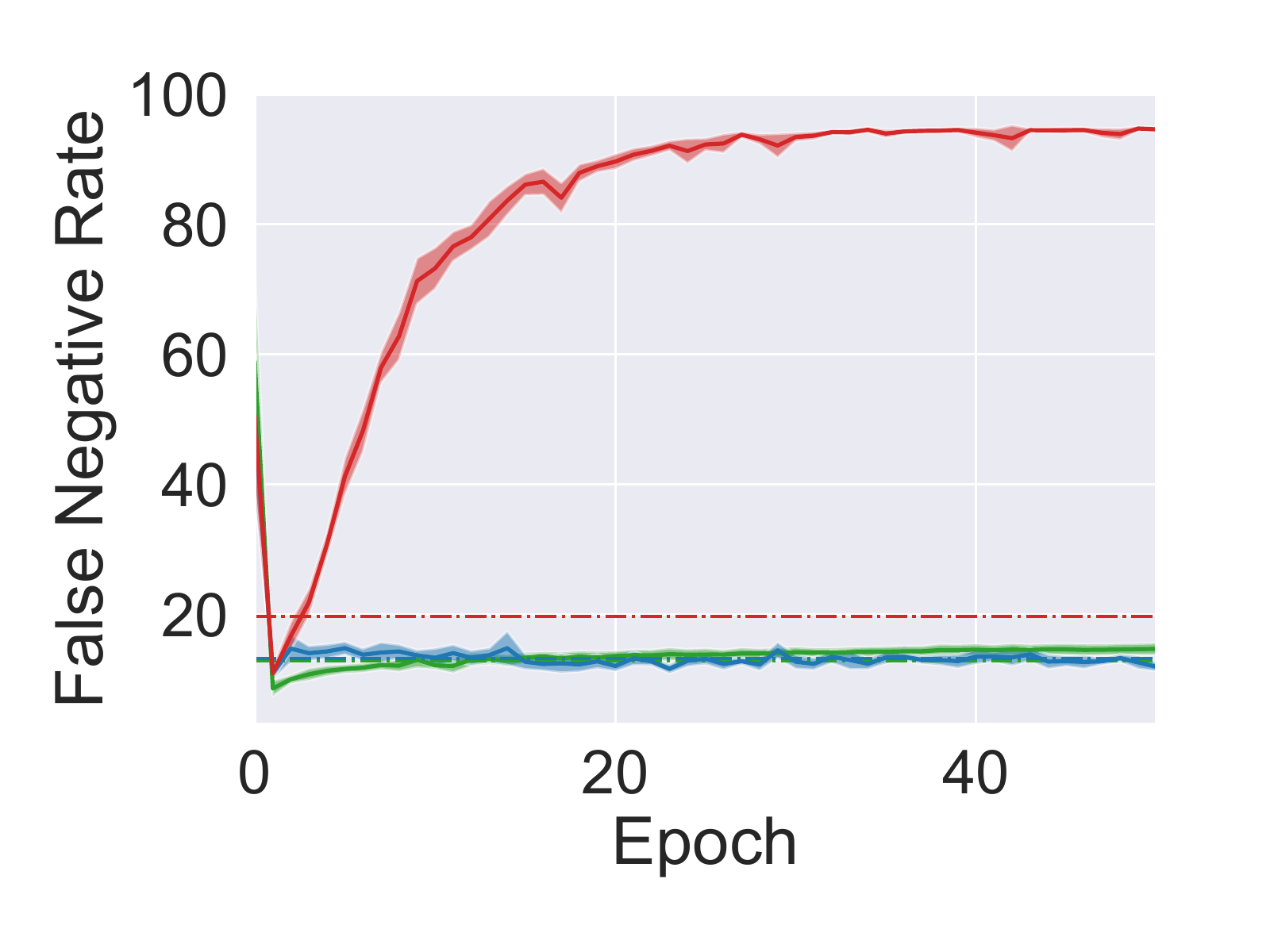}
    \includegraphics[width=0.32\linewidth]{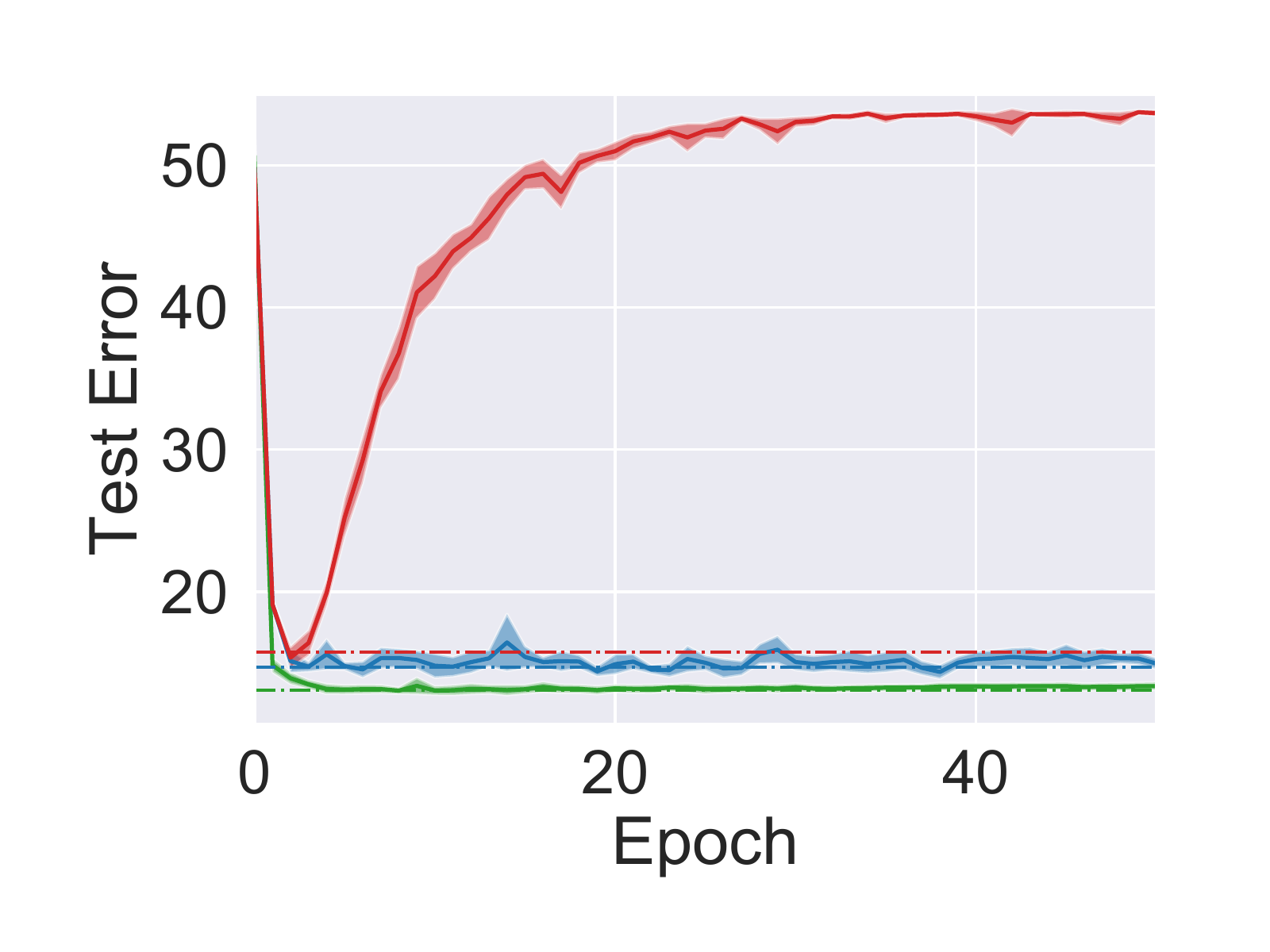}
    \vspace{-0.5em}
    \caption{20 Newsgroups}
  \end{subfigure}
  \caption{
    Comparison of uPU, nnPU and PUbN\textbackslash N over the four
    PU learning tasks.
    For each task, means and standard deviations are computed
    based on the same 10 random samplings.
    Dashed lines indicate the corresponding values of the final
    classifiers (recall that at the end we select the model
    with the lowest validation loss out of all epochs).
  }
  \label{fig: PU_comp}
\end{figure}

\subsection{Influence of $\eta$ and $\rho$}

In the proposed algorithm we introduce $\eta$ to control how
$\bar{R}_{s=-1}(g)$ is approximated from data and assume
that $\rho=p(y=-1, s=+1)$ is given.
Here we conduct experiments to see how our method
is affected by these two factors.
To assess the influence of $\eta$,
from \autoref{tab: effectiveness} we pick four
learning tasks and we choose $\tau$
from $\{0.5, 0.7, 0.9, 2\}$ while all the other hyperparameters
are fixed.
Similarly, to simulate the case where $\rho$ is misspecified,
we replace it by $\rho'\in\{0.8\rho, \rho, 1.2\rho\}$ in our
learning method and run experiments with all hyperparameters
being fixed to a certain value.
However, we still use the true $\rho$ to compute $\eta$
from $\tau$ to ensure that we always use the same number of
U samples in the second step of the algorithm
independent of the choice of $\rho'$.

\begin{table}[h!]
  \begin{center}
  \begin{threeparttable}
  \caption{
    Results on four different PUbN learning tasks when we vary
    the value of $\tau$ (and accordingly, $\eta$).
    Reported are means of false positive rates (FPR),
    false negative rates (FNR), misclassification rates (Error),
    and validation losses (VLoss) over 10 trials.\\[-0.2em]
  }
  \label{tab: eta}
    \begin{tabular}{llllrrrr}
      \toprule
      Dataset & P & biased N & $\tau$ &
      FPR & FNR & Error & VLoss \\
      \midrule
      \multirow{4}{*}{MNIST}
      & \multirow{4}{*}{0, 2, 4, 6, 8} 
      & \multirow{4}{*}{1, 3, 5} 
      & 0.5 & 4.79 & 4.32 & 4.56 & 10.11\\
      &&& 0.7 & 3.32 & 4.81 & 4.05 & \textbf{9.15} \\
      &&& 0.9 & 3.29 & 4.40 & \textbf{3.83} & 9.30 \\
      &&& 2 & 3.38 & 5.32 & 4.33 & 10.68 \\[0.2em]
      \cmidrule(l{2pt}r{2pt}){1-8} \\[-0.8em]
      \multirow{4}{*}{CIFAR-10}
      & \multirow{4}{2.5cm}{Airplane, automobile, ship, truck}
      & \multirow{4}{2.5cm}{Horse $>$ deer = frog $>$ others}
      & 0.5 & 8.31 & 12.35 & \textbf{9.92} & \textbf{12.50} \\
      &&& 0.7 & 8.23 & 13.15 & 10.20 & 12.62 \\
      &&& 0.9 & 7.54 & 14.68 & 10.40 & 13.08 \\
      &&& 2 & 6.23 & 20.29 & 11.85 & 13.64 \\[0.2em]
      \cmidrule(l{2pt}r{2pt}){1-8} \\[-0.8em]
      \multirow{4}{*}{CIFAR-10}
      & \multirow{4}{2.5cm}{Cat, deer, dog, horse}
      & \multirow{4}{2.5cm}{Bird, frog}
      & 0.5 & 14.45 & 27.57 & 19.70 & 22.08 \\
      &&& 0.7 & 13.20 & 27.27 & \textbf{18.83} & \textbf{20.72} \\
      &&& 0.9 & 13.00 & 32.61 & 20.84 & 23.78 \\
      &&& 2 & 11.67 & 31.49 & 19.60 & 22.52 \\[0.2em]
      \cmidrule(l{2pt}r{2pt}){1-8} \\[-0.8em]
      \multirow{4}{*}{20 Newsgroups}
      & \multirow{4}{2.5cm}{alt., comp., misc., rec.}
      & \multirow{4}{2.5cm}{soc. $>$ talk. $>$ sci.}
      & 0.5 & 11.28 & 12.90 & \textbf{12.18} & \textbf{16.04} \\
      &&& 0.7 & 11.40 & 13.58 & 12.62 & 16.64 \\
      &&& 0.9 & 10.09 & 16.70 & 13.79 & 16.90 \\
      &&& 2 & 10.34 & 20.55 & 16.06 & 20.99 \\
      \bottomrule
    \end{tabular}
  \end{threeparttable}
  \end{center}

  \begin{center}
  \begin{threeparttable}
  \caption{Mean and standard deviation of misclassification rates
    over 10 trials on different PUbN learning tasks when
    we replace $\rho$ by $\rho'\in\{0.8\rho, \rho, 1.2\rho\}$.
    Underlines indicate significant degradation of performance according
    to the 5\% t-test.\\[-0.2em]
  }
  \label{tab: rho}
    \begin{tabular}{llm{2.7cm}rrr}
      \toprule
      \multirow{2}{*}[-0.4em]{Dataset} 
      & \multirow{2}{*}[-0.4em]{P}
      & \multirow{2}{*}[-0.4em]{biased N}
      & \multicolumn{3}{c}{$\rho'/\rho$} \\[0.2em]
      \cmidrule(l{2pt}r{2pt}){4-6} \\[-1em]
      &&& \multicolumn{1}{r}{0.8}
      & \multicolumn{1}{r}{1}
      & \multicolumn{1}{r}{1.2} \\
      \midrule
      \multirow{2}{*}{MNIST}
      & \multirow{2}{*}{0, 2, 4, 6, 8}
      & 1, 3, 5
      &  $4.10\pm0.39$ & $4.05\pm0.27$ & $4.14\pm0.45$ \\
      && 9 $>$ 5 $>$ others
      &  $3.85\pm0.55$ & $3.91\pm0.66$ & $3.94\pm0.54$ \\
      \cmidrule(l{2pt}r{2pt}){1-6} \\[-0.8em]
      \multirow{2}{*}[-0.6em]{CIFAR-10}
      & \multirow{2}{2.5cm}{Airplane, automobile, ship, truck}
      & Cat, dog, horse
      & $10.23\pm0.59$ & $9.71\pm0.51$ & \underline{$10.32\pm0.57$} \\[0.3em]
      && Horse $>$ deer = frog $>$ others
      & $10.18\pm0.40$ & $9.92\pm0.42$ & $10.05\pm0.59$ \\
      \cmidrule(l{2pt}r{2pt}){1-6} \\[-0.8em]
      \multirow{2}{*}{CIFAR-10}
      & \multirow{2}{2.5cm}{Cat, deer, dog, horse}
      & Bird, frog
      & $18.94\pm0.50$ & $18.83\pm0.71$ & $19.06\pm0.80$ \\
      && Car, truck
      & $20.39\pm1.24$ & $20.19\pm1.06$ & $19.92\pm0.89$ \\[0.2em]
      \cmidrule(l{2pt}r{2pt}){1-6} \\[-0.8em]
      \multirow{3}{*}{20 Newsgroups}
      & \multirow{3}{2.5cm}{alt., comp., misc., rec.}
      & sci.
      & $13.49 \pm 0.61$ & $13.10 \pm 0.90$ & $13.31 \pm 1.05$ \\
      && talk.
      & $12.64 \pm 0.69$ & $12.61 \pm 0.75$ & \underline{$13.77 \pm 0.85$} \\
      && soc. $>$ talk. $>$ sci.
      & \underline{$12.90 \pm 0.79$}
      & $12.18 \pm 0.59$ & \underline{$12.74 \pm 0.35$} \\
      \bottomrule
    \end{tabular}
  \end{threeparttable}
  \end{center}
\end{table}

The results are reported in \autoref{tab: eta} and \autoref{tab: rho}.
We can see that the performance of the algorithm is sensitive
to the choice of $\tau$.
With larger value of $\tau$, more U data are treated as N data
in PUbN learning, and consequently it
often leads to higher false negative rate
and lower false positive rate.
The trade-off between these two measures is a classic problem
in binary classification.
In particular, when $\tau=2$, a lot more U samples are
involved in the computation of the PUbN risk
\eqref{eq: PUbN risk estimator}, but this does not allow the
classifier to achieve a better performance.
We also observe that there is a positive correlation between
the misclassification rate and the validation loss, which
confirms that the optimal value of $\eta$ can be chosen
without need of unbiased N data.

\autoref{tab: rho} shows that in general slight misspecification
of $\rho$ does not cause obvious degradation of the classification performance.
In fact, misspecification of $\rho$ mainly affect the
weights of each sample when we compute
$\hat{R}_{\textrm{PUbN}, \eta, \hat{\sigma}}$
(due to the direct presence of $\rho$ in
\eqref{eq: PUbN risk estimator} and influence on estimating $\sigma$).
However, as long as the variation of these weights remain
in a reasonable range, the learning algorithm should yield
classifiers with similar performances.

\subsection{Estimating $\sigma$ from Separate Data} \label{app: sig sep}

Theorem \ref{th: bound} suggests that $\hat{\sigma}$ should be
independent from the data used to compute
$\hat{R}_{\textrm{PUbN}, \eta, \hat{\sigma}}$.
Therefore, here we investigate the performance of our algorithm
when $\hat{\sigma}$ and $g$ are optimized using different sets
of data.
We sample two training sets and two validation sets
in such a way that they are all disjoint.
The size of a single training set and a single validation set
is as indicated in Appendix \ref{app: tvt}, except for
20 Newsgroups we reduce the number of examples in a single set
by half.
We then use different pairs of training and validation sets
to learn $\hat{\sigma}$ and $g$.
For 20 Newsgroups we also conduct standard experiments
where $\hat{\sigma}$ and $g$ are learned on the same data,
whereas for MNIST and CIFAR-10 we resort to
\autoref{tab: effectiveness}.

The results are presented in \autoref{tab: sep}.
Estimating $\sigma$ from separate data does not seem to benefit much
the final classification performance, despite
the fact that it requires collecting twice more samples.
In fact, $\hat{\bar{R}}_{s=-1,\eta,\hat{\sigma}}^-(g)$
is a good approximation of
$\bar{R}_{s=-1,\eta,\hat{\sigma}}^-(g)$
as long as the function $\hat{\sigma}$ is smooth
enough and does not possess abrupt changes between
data points.
With the use of non-negative correction, validation data
and L2 regularization, the resulting $\hat{\sigma}$
does not overfit training data so this should always be
the case.
As a consequence, even if $\hat{\sigma}$ and $g$ are learned
on the same data, we are still able to achieve small generalization
error with sufficient number of samples.

\begin{table}[t]
  \begin{center}
  \begin{threeparttable}
  \caption{Mean and standard deviation of misclassification rates
    over 10 trials on different PUbN learning tasks with
    $\hat{\sigma}$ and $g$ trained using either the same or
    different sets of data.\\[-0.2em]
  }
  \label{tab: sep}
    \begin{tabular}{llm{2.7cm}rr}
      \toprule
      \multirow{2}{*}[-0.3em]{Dataset} 
      & \multirow{2}{*}[-0.3em]{P}
      & \multirow{2}{*}[-0.3em]{biased N}
      & \multicolumn{2}{c}{Data for $\hat{\sigma}$ and $g$} \\
      \cmidrule(l{3pt}r{3pt}){4-5} \\[-1em]
      &&& \multicolumn{1}{c}{Same}
      & \multicolumn{1}{c}{Different} \\
      \midrule\\[-1em]
      \multirow{2}{*}{MNIST}
      & \multirow{2}{*}{0, 2, 4, 6, 8}
      & 1, 3, 5 & $4.05\pm0.27$ & $3.71\pm0.45$ \\
      && 9 $>$ 5 $>$ others & $3.91\pm0.66$ & $4.06\pm0.36$ \\[0.2em]
      \cmidrule(l{2pt}r{2pt}){1-5} \\[-0.8em]
      \multirow{2}{*}[-0.6em]{CIFAR-10}
      & \multirow{2}{2.5cm}{Airplane, automobile, ship, truck}
      & Cat, dog, horse
      & $9.71\pm0.51$ & $10.00\pm0.51$ \\[0.3em]
      && Horse $>$ deer = frog $>$ others
      & $9.92\pm0.42$ & $9.66\pm0.46$ \\
      \cmidrule(l{2pt}r{2pt}){1-5} \\[-0.8em]
      \multirow{2}{*}{CIFAR-10}
      & \multirow{2}{2.5cm}{Cat, deer, dog, horse}
      & Bird, frog
      & $18.83\pm0.71$ & $18.52\pm 0.70$ \\
      && Car, truck
      & $20.19\pm1.06$ & $19.98\pm0.93$ \\[0.2em]
      \cmidrule(l{2pt}r{2pt}){1-5} \\[-0.8em]
      \multirow{3}{*}{20 Newsgroups}
      & \multirow{3}{2.5cm}{alt., comp., misc., rec.}
      & sci.
      & $15.61 \pm 1.50$ & $16.60 \pm 2.38$ \\
      && talk.
      & $17.14 \pm 1.87$ & $15.80 \pm 0.95$ \\
      && soc. $>$ talk. $>$ sci.
      & $15.93 \pm 1.88$ & $15.80 \pm 1.91$ \\
      \bottomrule
    \end{tabular}
  \end{threeparttable}
  \end{center}
\end{table}

\begin{table}[h!]
  \begin{center}
  \begin{threeparttable}
  \caption{Mean and standard deviation of misclassification rates
    over 10 trials on different PUbN learning tasks
    for the two possible definitions of the nnPNU algorithm.\\[-0.2em]
  }
  \label{tab: nnPNU-alt}
    \begin{tabular}{llm{2.7cm}rr}
      \toprule
      Dataset & P & biased N & nnPNU & nnPU + PN \\
      \midrule\\[-1em]
      \multirow{2}{*}{MNIST}
      & \multirow{2}{*}{0, 2, 4, 6, 8}
      & 1, 3, 5 & $5.33\pm0.97$ & $5.68\pm0.78$ \\
      && 9 $>$ 5 $>$ others & $4.60\pm0.65$ & $5.10\pm1.54$ \\[0.2em]
      \cmidrule(l{2pt}r{2pt}){1-5} \\[-0.8em]
      \multirow{2}{*}[-0.6em]{CIFAR-10}
      & \multirow{2}{2.5cm}{Airplane, automobile, ship, truck}
      & Cat, dog, horse
      & $10.25\pm0.38$ & $10.87\pm0.62$ \\[0.3em]
      && Horse $>$ deer = frog $>$ others
      & $9.98\pm0.53$ & $10.77\pm0.65$ \\
      \cmidrule(l{2pt}r{2pt}){1-5} \\[-0.8em]
      \multirow{2}{*}{CIFAR-10}
      & \multirow{2}{2.5cm}{Cat, deer, dog, horse}
      & Bird, frog
      & $22.00\pm0.53$ & $21.41\pm 1.01$ \\
      && Car, truck
      & $22.00\pm0.74$ & $21.80\pm0.74$ \\[0.2em]
      \cmidrule(l{2pt}r{2pt}){1-5} \\[-0.8em]
      \multirow{3}{*}{20 Newsgroups}
      & \multirow{3}{2.5cm}{alt., comp., misc., rec.}
      & sci.
      & $14.69 \pm 0.46$ & $14.50 \pm 1.32$ \\
      && talk.
      & $14.38 \pm 0.74$ & $14.71 \pm 1.01$ \\
      && soc. $>$ talk. $>$ sci.
      & $14.41 \pm 0.70$ & $13.66 \pm 0.72$ \\
      \bottomrule
    \end{tabular}
  \end{threeparttable}
  \end{center}
\end{table}

\subsection{Alternative Definition of nnPNU} \label{app: pnu-alt}

In subsection \ref{subsec: pnu}, we define the nnPNU algorithm
by forcing the estimator of the whole N partial risk
to be positive.
However, notice that the term $\gamma(1-\pi)\hat{R}^-\sub{N}(g)$
is always positive and the chances are that including it simply
makes non-negative correction weaker and is thus
harmful to the final classification performance.
Therefore, here we consider an alternative definition of nnPNU
where we only force the term
$(1-\gamma)(\hat{R}\sub{U}^-(g)-\pi\hat{R}\sub{P}^-(g))$
to be positive.
We plug the resulting algorithm in the experiments
of subsection \ref{subsec: effectiveness} and summarize the results
in \autoref{tab: nnPNU-alt} in which
we denote the alternative version of nnPNU
by nnPU+PN since it uses the same non-negative
correction as nnPU.
The table indicates that neither of the two definitions of nnPNU
consistently outperforms the other.
It also ensures that there is always a clear superiority of
our proposed PUbN algorithm compared to nnPNU
despite its possible variant that is considered here.

\subsection{More on Text Classification} \label{app: cbs}

\citet{fei2015social} introduced CBS learning
in the context of text classification.
The idea is to transform document representation from the traditional
n-gram feature space to a center-based similarity (CBS) space,
in hope that this could
mitigate the adverse effect of N data being biased.
They conducted experiments with SVMs and showed that the transformation
could effectively help improving the classification performance.
However, this process largely reduces the number of input features,
and for us it is unclear whether CBS transformation would still be
beneficial when it is possible to use other kinds of text features
or more sophisticated models.
On the contrary, we propose a training strategy that is a priori
compatible with any extracted features and models.
As mentioned in the introduction, another important difference between
our work and CBS learning
is that the latter does not assume availability of U data while
the presence of U data is indispensable in our case.

\textbf{Experimental Setup.}\enspace
Below we compare PUbN learning with
CBS learning on 20 Newsgroups text classification experiments.
The numbers of different types of examples that are used by
each learning method
are summarized in \autoref{tab: cbs-size},
where PbN denotes the case where only P
and bN data are available.
Notice that here we consider two different PbN learning settings,
depending on the number of used bN samples.
If we compare the two PbN settings with the PUbN setting, for one
we add extra U samples and for the other we replace a part of
bN samples by U samples.
A second point to notice is that no validation data are used in the
PbN settings.
In fact, although we empirically observe that the performance of
CBS learning is greatly influenced by the choice of the
hyperparameters, from \cite{fei2015social} it is unclear how
validation data can be used for hyperparameter tuning.
As a result, for CBS learning we simply report the best results
that were achieved in our experiments.
The reported values can be regarded as an upper bound
on the performance of this method.
By the way, SVM training itself usually does not require
the use of validation data. 

To make PUbN and CBS learning comparable, we use linear model and take
normalized tf-idf vectors as input variables in PUbN learning.
We also include another PbN baseline that directly trains
a SVM on the normalized tf-idf feature vectors.
Regarding CBS learning, we use Chi-Square feature selection
as it empirically produces the best results.
The number of retained features is considered as a hyperparameter
and as mentioned above its value can in effect have
great impact on the final result.
Although \citet{fei2015social} suggested using simultaneously
unigram, bigram and trigram representations of each document,
the use of bigrams and trigrams does not appear to provide
any benefits in our experiments.
Therefore for the results that are presented here we only use
the unigram representation of a document.

\textbf{Results.}\enspace
\autoref{tab: cbs-res} affirms the superiority of PUbN learning,
even in the case where the PbN and PUbN settings share
the same total number of samples.
Only in the third learning task with 1000 bN samples and
without CBS transformation PbN learning slightly
outperforms PUbN learning.
This can be explained by the fact that in this learning task,
though the collected N samples are biased, they still cover
all the possible topics appearing in the N distribution.

CBS transformation does sometimes improve the classification
accuracy, but the improvement is not consistent.
We conjecture that CBS learning is not so effective here
both because our P class contains several distinct topics,
and because our N class is not so diverse compared with \cite{fei2015social}.
Having multiple topics in P implies that there may not
be a meaningful center for the P class in the feature space.
On the other hand, in the results reported by \citet{fei2015social}
we can see that the benefit of CBS learning becomes
less significant when a large proportion of N topics can be
found in the bN set.
In particular, in the last learning task, CBS transformation
even turns out to be harmful.
We also observe that after CBS transformation, the classifier
becomes less sensitive to both how bN data are sampled and
the size of the bN set.
This seems to suggest that CBS learning is the most beneficial
when both the number of topics appeared in the bN set and the
amount of bN data are very limited.

\begin{table}[t]
  \begin{center}
  \begin{threeparttable}
  \caption{Number of samples in each set under different learning
    settings for supplementary 20 Newsgroups experiments that
    compare PUbN learning with PbN learning and \cite{fei2015social}.\\[-0.2em]}
  \label{tab: cbs-size}
  \begin{tabular}{lcccccc@{\hspace{1em}}|>{\centering\arraybackslash}p{1.5cm}}
      \toprule
      & P train & bN train & U train
      & P validation & bN validation & U validation & Total
      \\
      \midrule\\[-1em]
      PUbN
      & 500 & 320 & 500 & 100 & 80 & 100 & 1600\\
      PbN 400
      & 600 & 400 & NA & NA & NA & NA & 1000\\
      PbN 1000
      & 600 & 1000 & NA & NA & NA & NA & 1600\\
      \bottomrule
    \end{tabular}
  \end{threeparttable}
  \end{center}
\end{table}

\begin{table}[t]
  \begin{center}
  \begin{threeparttable}
  \caption{Mean and standard deviation of misclassification rates
    over 10 trials for text classification tasks on 20 newsgroups
    to compare PUbN learning with CBS learning and a PbN baseline.\\[-0.2em]
  }
  \label{tab: cbs-res}
    \begin{tabular}{lrrrrr}
      \toprule
      \multirow{2}{*}[-0.3em]{biased N}
      & \multicolumn{1}{c}{\multirow{2}{*}[-0.3em]{PUbN}}
      & \multicolumn{2}{c}{PbN 400} 
      & \multicolumn{2}{c}{PbN 1000}
      \\
      \cmidrule(l{3pt}r{3pt}){3-4}
      \cmidrule(l{3pt}r{3pt}){5-6} \\[-1em]
      && tf-idf & CBS & tf-idf & CBS \\
      \midrule\\[-1em]
      sci.
      & $\mathbf{13.06 \pm 0.69}$
      & $25.31 \pm 0.72$ & $21.15 \pm 0.73$
      & $19.68 \pm 0.94$ & $17.72 \pm 0.74$ \\
      talk.
      & $\mathbf{14.88 \pm 0.63}$
      & $23.42 \pm 0.44$ & $22.73 \pm 0.81$
      & $19.41 \pm 0.50$ & $19.59 \pm 0.47$ \\
      soc. $>$ talk. $>$ sci.
      & $\mathbf{13.96 \pm 0.61}$
      & $19.82 \pm 0.86$ & $21.68 \pm 0.84$
      & $\mathbf{13.68 \pm 0.43}$ & $17.43 \pm 0.53$ \\
      \bottomrule
    \end{tabular}
  \end{threeparttable}
  \end{center}
\end{table}


\end{document}